%% file: main.tex
\PassOptionsToPackage{dvipsnames,table}{xcolor}
\documentclass[11pt,letterpaper]{article}
\usepackage[margin=1in]{geometry}
\usepackage{mathpazo}
\usepackage[T1]{fontenc} %

\newcommand{\Unif}{\mathrm{Unif}}

\title{Proper Positive Learning}

\input{macros}

\makeatletter
\renewcommand\paragraph{\@startsection{paragraph}{4}{\z@}%
  {\smallskipamount}%
  {-1em}%
  {\normalfont\normalsize\bfseries}}
\makeatother
\hypersetup{
  colorlinks = true,
  urlcolor = {blueGrotto},
  linkcolor = {royalBlue},
  citecolor = {navyBlue}
}
\usepackage[
  backref=true,
  backend=biber,
  natbib=true,
  style=alphabetic,
  sorting=alphabeticlabel,
  sortcites=true,
  minbibnames=3,
  maxbibnames=999,
  mincitenames=4,
  maxcitenames=4,
  minalphanames=4,
  maxalphanames=4,
  doi=false,
  isbn=false,
]{biblatex}
\DeclareSortingTemplate{alphabeticlabel}{
  \sort[final]{%
    \field{labelalpha}
  }
  \sort{%
    \field{year}
  }
  \sort{%
    \field{title}
  }
}
\AtEveryBibitem{%
  \clearname{editor}%
}
\AtEveryBibitem{%
  \clearlist{location}%
}

\AtBeginRefsection{\GenRefcontextData{sorting=ynt}}
\AtEveryCite{\localrefcontext[sorting=ynt]}
\usepackage{soul}
\usepackage[framemethod=TikZ]{mdframed}
\mdfsetup{%
backgroundcolor=royalBlue!0, %
roundcorner=4pt,
skipabove=5pt,
skipbelow=5pt,
linewidth=1pt}

\newcommand{\mpos}{m} %

\newcommand{\highlight}[1]{\textcolor{royalBlue}{\textit{#1}}}

\setul{2pt}{0.4pt}

\newcommand{\paperTitle}{Surprises in Proper Positive-Only Learning}

\title{\paperTitle{}}
\author{
    \begin{tabular}{@{}C{7.8cm}@{\hspace{0.3cm}}C{7.8cm}@{}}
        {\bf Shai Ben-David}
            & {\bf Farnam Mansouri}\\
        {\small University of Waterloo and Vector Institute}
            & {\small University of Waterloo and Vector Institute}\\
        {\small\mbox{{\href{mailto:shai@uwaterloo.ca}{shai@uwaterloo.ca}}}}
            & {\small\mbox{{\href{mailto:f5mansou@uwaterloo.ca}{f5mansou@uwaterloo.ca}}}}\\[4mm]
        {\bf Anay Mehrotra}
            & {\bf Manolis Zampetakis}\\
        {\small Stanford University}
            & {\small Yale University}\\
        {\small\mbox{{\href{mailto:anaymehrotra1@gmail.com}{anaymehrotra1@gmail.com}}}}
            & {\small\mbox{{\href{mailto:manolis.zampetakis@yale.edu}{manolis.zampetakis@yale.edu}}}}
    \end{tabular}
}
\date{}

\begin{document}

\hypersetup{pageanchor=false}
\pagenumbering{gobble}
\maketitle

\begin{abstract}
 
    Binary classification from positive-only samples is a variant of PAC learning in which the learner receives i.i.d.\ samples from the positive region of an unknown target concept, but is evaluated under the original distribution (which places mass on both positive and negative regions). 
    This model dates back to \cite[STOC]{natarajan1987learning}, and the characterization of improper learning is well-known -- it even appears in textbooks \citep[Exercise 3.7]{kearns1994introduction}.
    The characterization of \textit{proper} positive-only learning, however, has long remained open.
    In this work, we revisit and settle this question: a concept class is properly learnable from positive-only samples if and only if it has finite VC dimension and satisfies a new combinatorial condition, which we call uniform exterior separability. 
    Together with several separation results, this characterization reveals a surprisingly rich landscape that differs sharply from standard PAC learning: proper and improper learning are separated, randomized and deterministic proper learning are separated, there are classes for which \textit{no} ERM is a learner, and finite VC dimension does not suffice even for non-uniform learning.
    Along the way, we introduce new combinatorial dimensions that we believe can be of broader interest in learning theory.
\end{abstract}

\hypersetup{pageanchor=true}
\pagenumbering{arabic}

\section{Introduction}
\label{sec:introduction}

In the celebrated PAC learning model \citep{valiant1984theory}, a learner observes positive and negative examples drawn from an unknown distribution and must output a classifier with small misclassification error. %
This formulation assumes that the learner has access to labeled samples from both sides of the target concept.
However, in practical scenarios, this assumption is often violated because negative examples are expensive to collect, scarce, or completely absent. 
This issue was already motivated by Valiant, who wrote: \emph{``While it may be reasonable to discuss the distribution of the attributes of elephants, we may prefer not discussing the distribution of the attributes of non-elephants'' --- \citep{valiant1984deductive}.}
It also arises in many modern applications: for instance, in medical diagnosis, confirmed cases of a condition are recorded but healthy controls may be unlabeled \citep{elkan2008learning}; in web search and recommendation systems, engagement provides positive feedback while non-engagement remains ambiguous \citep{liu2003building}.
These practical challenges can be modeled through a natural extension of the PAC framework, where the learner receives positive samples, but receives \textit{no} negative samples.

Let $\cD$ be an unknown distribution over the feature space $\cX$, and let $\hstar \in \hyH$ be the target concept.
The learner receives a multiset of $n$ \iid{} samples from the conditional distribution $\cD_+ = \cD \mid (\hstar(X) = 1)$, and outputs a hypothesis whose error is measured under the original distribution $\cD$.
\begin{definition}[PAC Learning with Positive-Only Samples; cf.\ {\citep[Exercise 3.7]{kearns1994introduction}}]
    \label{def:positiveOnly}
    A class $\hyH$ is \emph{learnable from positive-only samples} if there exist a sample-complexity function $n\colon (0,1)^2 \to \N$ and a learner $\cA$ such that, for every accuracy and confidence parameters $\eps,\delta \in (0,1)$, every distribution $\cD$ over $\cX$, and every target $\hstar \in \hyH$ with $\Pr_{X \sim \cD}[\hstar(X) = 1] > 0$, given $n(\eps,\delta)$ \iid{} samples from $\cD_+$, $\cA$ outputs a hypothesis $\wh{h}$ satisfying
    \[
        \Pr\nolimits_{X \sim \cD}\sinparen{\wh{h}(X) \neq \hstar(X)} \leq \eps\,,
    \]
    with probability at least $1-\delta$ over the sample and the learner's internal randomness.
\end{definition}
\noindent Note that although the learner only sees samples from $\cD_+$, the hypothesis $\wh{h}$ is evaluated under the full distribution $\cD$, which has mass on both the positive and the negative regions.
The positive samples identify points that must be labeled positive, but they give no information about the negative region, where the learner must still avoid too many false positives.

A learner is said to be \emph{proper} if it always outputs a hypothesis in $\hyH$, and \emph{improper} otherwise (if it can output a hypothesis that is outside of $\hyH$).
In ordinary PAC learning, both types of learners are equally powerful: a class is PAC learnable if and only if it has finite VC dimension $\vc(\hyH) < \infty$.
Moreover, whenever $\vc(\hyH) < \infty$, \textit{every} empirical risk minimization (ERM) rule, which is any rule that returns a hypothesis in $\hyH$ consistent with the labeled samples, is a PAC learning algorithm.\footnote{While every ERM is a PAC learner, ERM rules, in general, do not attain the optimal sample complexity \citep{hanneke2016optimal,hanneke2016refined,larsen2023bagging}.}

For positive-only learning, the improper version admits a clean characterization:
Let $\hyH_\cap$ denote the class of finite intersections of concepts in $\hyH$.
Then $\hyH$ is improperly learnable from positive-only samples if and only if $\vc(\hyH_\cap) < \infty$ \citep{kearns1994introduction}.
\citet{graus1989lower} gives an alternate combinatorial characterization, which even appears as a textbook exercise \citep[Exercise 3.7]{kearns1994introduction}.
Recent work of \citet{hanneke2024star}, as a corollary, pins down the tight sample complexity of improper positive-only learning via a dimension called the one-centered star number; see \citet{anay2026thesis} for an overview.
However, despite the improper case being very well understood, and there being decades of work on various variants of PAC learning \citep{haussler1990survey}, including variants of positive-only learning \citep{natarajan1987learning,shvaytser1990positiveonly,kivinen1995learning,FJK96,NR09,AGR13,DDS15,Kontonis2019EfficientTS,CDS20,mansouri2025learning,lee2026smoothed}, \mbox{the following natural question has remained unresolved:}

    \begin{center}
        \emph{What is the characterization of learnability for proper positive-only PAC learning?}
    \end{center}

\vspace{-2mm}

\paragraph{Our contributions.}
In this work, we settle this question, and show that proper positive-only learning has a very rich structure.
We make the following contributions. 
\vspace{-1mm}

\begin{enumerate}[itemsep=-1pt,leftmargin=15pt]
    \item \textbf{A characterization of proper positive-only learning (\cref{thm:main}).}
    We prove that $\hyH$ is properly learnable from positive-only samples if and only if $\vc(\hyH) < \infty$ and $\hyH$ satisfies a new combinatorial condition, \emph{uniform exterior separability} (\cref{def:ues}), which controls false positives in the parts of the input space not forced to be positive by the sample.
    \item \textbf{Separations from ordinary PAC learning (\cref{thm:separation-improper,thm:randomness-necessary,thm:ERM-separation,thm:nonuniform-separation}).}
    Together with several separation results, our characterization shows that the landscape of learnability for proper positive-only learning is considerably different from that of ordinary PAC learning.
    In particular:
    \begin{enumerate}[itemsep=-1pt]
        \item Proper and improper learning are separated (\cref{thm:separation-improper});
        \item Randomized and deterministic proper learning are separated (\cref{thm:randomness-necessary});
        \item There exist classes for which no deterministic ERM rule is a learner (\cref{thm:ERM-separation}); and 
        \item Finite VC dimension does not suffice even for non-uniform learning (\cref{thm:nonuniform-separation}).
    \end{enumerate}
    \item \textbf{Randomness as a resource (\cref{thm:few-random-bits-suffice,thm:randomness-necessary}).}
    We quantify the randomness needed for proper positive-only learning.
    On the positive side, we show that very few random bits suffice: if $\hyH$ admits a randomized proper learner from positive-only samples, then it admits one that uses only $\wt{O}\!\inparen{d + \log \nfrac{1}{\eps}}$ random bits (where $d = \vc(\hyH)$) and $\wt{O}\!\inparen{\nfrac{d}{\eps}}$ positive samples (\cref{thm:few-random-bits-suffice}).
    For classes of constant VC dimension (\ie{}, $d=O(1)$), this is exponentially fewer random bits than samples.
    On the negative side, as mentioned above, randomness cannot be eliminated entirely: there exist classes that are properly learnable by a randomized learner but by no deterministic one (\cref{thm:randomness-necessary}).
    \item \textbf{A characterization of stable proper positive-only learning (\cref{thm:stable-characterization}).}
    We also study a notion of \emph{stability} for proper learners, which requires the learner's output to remain unchanged when given additional samples consistent with its current hypothesis (\cref{def:stable}), in the spirit of stable sample compression schemes \citep{bousquet2020proper}.
    We show that $\hyH$ is properly learnable by a stable 
    learner if and only if $\vc(\hyH) < \infty$ and $\hyH$ satisfies \emph{exact exterior separation} (\cref{def:ees}), a strict strengthening of uniform exterior separability (\cref{thm:stable-characterization}).
\end{enumerate}
\noindent More broadly, we view our results as part of an effort to understand how the right combinatorial dimension changes under \mbox{different supervision models}. The dimensions and conditions we introduce here, particularly the \mbox{exterior separability} notions, do not coincide with VC dimension and other standard parameters, and we expect they will find use beyond \mbox{positive-only learning}, in settings where the learner must extrapolate from \mbox{limited supervision}.

\subsection{Related Work}
\label{sec:related-works}
We have already discussed key related work on learning from positive-only samples when motivating the model.
Here, we situate our results within several neighboring literature.

\smallskip
\noindent\textbf{Learning with One-Sided Error.}
Positive-only learning has also been studied under the additional restriction that the learner makes no false-positive errors, known as the \emph{one-sided error} model.
\citet{natarajan1987learning} initiated this line of work and characterized proper learnability in the one-sided error setting.
Subsequently \citet{shvaytser1990positiveonly,kivinen1995learning} characterized improper learning and bounded its sample complexity.
Natarajan conjectured that his characterization extends to the (two-sided) proper positive-only learning model studied in this paper.\footnote{In support, he analyzed the two-sided problem under the uniform distribution, remarking that ``if the distribution is allowed to be arbitrary, we can show easily that any learning algorithm can be forced to deviate from the function to be learned with probability approaching unity.''}
\highlight{We disprove this conjecture in \cref{thm:separation-natarajan} (which, in turn, utilizes our main characterization; \cref{thm:main}).}

\smallskip
\noindent\textbf{Distribution-Specific Positive-Only Learning.}
Several works study positive-only learning under specific distributional assumptions.
When the distribution is Gaussian, \citet{Kontonis2019EfficientTS,lee2024efficient} design computationally efficient algorithms with guarantees for concept classes with bounded {Gaussian surface area}.
For unknown distributions satisfying a mild smoothness condition, \citet{lee2026smoothed} gave methods for learning all finite-VC classes as well as classes approximable by polynomials, and \citet{kouridakis2026truncation} obtained a polynomial-time algorithm for finite-VC classes on the real line.
A separate line studies positive-only learning when the distribution of positive examples is explicitly given to the learner \citep{FJK96,NR09,RG09,AGR13,DDS15,CDS20,JKV23}.
In contrast to all these works, and as is standard in statistical learning theory, we do not attempt to design computationally efficient algorithms and, instead, obtain characterizations of learnability, with a particular focus on \emph{proper} learners.

\smallskip
\noindent\textbf{Learning from Positive and Unlabeled Examples.}
The \emph{positive-and-unlabeled} (PU) model gives the learner access to unlabeled samples in addition to positive examples \citep{denis1998pac,liu2002partially,elkan2008learning,blanchard2010semi,du2015convex,bekker2018estimating,lee2024efficient,mansouri2025learning,lee2026smoothed}.
Especially relevant to our work, \citet{mansouri2025learning} proved lower bounds on the number of unlabeled samples required for PU learning, demonstrating a separation between \mbox{PU learning and learning from positive examples alone.}

\smallskip
\noindent\textbf{Out-of-Distribution Detection.}
The problem of learning from positive samples is also closely related to theoretical work on
out-of-distribution detection.
In the binary case, OOD detection can be viewed as learning from positive-only examples.
The results of \citet{NEURIPS2022_f0e91b13,JMLR:v25:23-1257} imply that consistency is impossible
in full generality in the agnostic setup, and identify conditions under which consistency can be
recovered. 
While in some of our results, we also consider consistency, our focus is on the stronger requirement of bounding the sample complexity of learning instead of just consistency.

\smallskip
\noindent\textbf{Language Identification in the Limit.}
The line of work on language identification in the limit, \eg{}, \citet{gold1967language,angluin1979finding,angluin1988learning} can also be viewed as learning from positive examples only, but with a different success criterion: the sample size is allowed to depend on the target concept and on the order in which positive examples are presented.
Thus, this line of work is closer in spirit to non-uniform learnability than to the uniform PAC-style guarantees studied here.
Stochastic versions of these frameworks have also been studied \cite{angluin1988identifying,kalavasis2025limits,hogsgaard2026agnostic} in the universal-learning formulation of \cite{bousquet2021theory}, which is quite different from the PAC learning model that we study here.

\section{Preliminaries}
    \label{sec:preliminaries}
    In this section, we introduce basic notation and preliminaries.%

\paragraph{Notation.}\hspace{-2mm}
    We use $f\lesssim g$ and $f\simeq g$ to denote $f=O(g)$ and $f=\Theta(g)$ respectively.
    We use $\widetilde{O}(\cdot)$, $\widetilde{\Omega}(\cdot)$, and $\widetilde{\Theta}(\cdot)$ to hide poly-logarithmic factors.    
    Further, for a set $A$, we use $A^{[*]}$ to denote the set of all finite multi-sets of members of $A$, and $A^*$ as the set of all finite sequences of elements of $A$. (A multi-set contains repetitions of instances, but not their order. A sequence also records their order.)

\paragraph{Positive-only learning.}
    Let $\hyH$ denote a concept class over a domain $\cX$. (Formally, $\hyH$ is a set of subsets of $\cX$.) We identify each concept with the set of points that it labels positive. A realizable instance of positive-only learning is specified by an unknown distribution $\cD$ over $\cX$ and an unknown target concept $\ell \in \cH$ with $\cD(\ell)>0$. 
    For measurable $A \subseteq \cX$, we define %
\[
\mathcal{D}_+(A) \coloneqq \mathcal{D}\!\inparen{A \mid x \in \ell}
\qquadand
\mathcal{D}_-(A) \coloneqq  \mathcal{D}\!\inparen{A \mid x \notin \ell}
\]
whenever the conditioning is well-defined. 
The learner observes only positive examples, \ie{}, an i.i.d.\ sample from $\cD_+$.
A hypothesis is a subset $h \subseteq \cX$ and its loss or \textit{error} is evaluated under the full distribution $\mathcal{D}$:
\[
\err_{\mathcal{D}}\!\inparen{h,\ell} \coloneqq  \mathcal{D}\!\inparen{h \triangle \ell}
    \quadtext{for} h \subseteq \mathcal{X}\,.
\]\vspace{-3mm}

\paragraph{Learning algorithms and sample complexity.}
    A proper learning algorithm, or simply \textit{a proper learner} $\cA$,  is a function $\cA  \colon \cX^{[*]} \times \Omega \to \cH$ where $\Omega=\zo^{\N}$ provides a sequence of (potentially) random bits. For a fixed sample $\sampleP \in \cX^*$, we use $\cA(\sampleP)$ as shorthand for the random variable $\omega \mapsto \cA(\sampleP,\omega)$.
    We say that a concept class $\cH$ is learnable from positive-only samples by a proper learner $\cA$ if there exists a function
    $\mpos{}_{\cH}: (0, 1) \times (0, 1) \rightarrow \Nat$ so that for every distribution $\cD$ over $\cX$ and every $\ell \in \cH$, every accuracy and confidence $\eps,\delta \in \inparen{0,1}$, and every number $n \geq \mpos{}_{\cH}\!\inparen{\eps,\delta}$, with probability at least $1-\delta$, the learner has an error at most $\eps$, \ie{}, 
    \[
    \Pr\nolimits_{\sampleP \sim \cD_+^n, \cA}\inparen{
    \err_{\cD}\inparen{\cA \inparen{\sampleP},\ell} \leq \eps
    } \geq 1-\delta\,.
    \]
    \vspace{-3mm}

\paragraph{Version space and closure.}
    The concept of a closure of a sample set is frequently used in learning theory and will also be central in our characterizations. 
    Consider a finite sample of ``positive'' points $S \subseteq \mathcal{X}$ realized by $\cH$. 
    (Formally, $S$ is a positive sample set realizable by $\cH$ if there is some $h\in \hyH$ such that $S\subseteq h$.)
    The version space $\hyH_S$ of $S$ is the set of all hypotheses in $\hyH$ consistent with $S$.
    In other words, it is the set of all hypotheses that could be the target concept.
    Formally, 
    \[
    \cH_S \coloneqq  \inbrace{h \in \cH : S \subseteq h}\,.\tag{Version Space}
    \]
    \noindent The closure of $S$ with respect to $\hyH$ is the intersection of all hypotheses $h$ in the version space, \ie{}, 
    \[
    \Cl_{\cH}\!\inparen{S} \coloneqq  \bigcap\nolimits_{h \in \cH_S} h\,. \tag{Closure}
    \]
    \noindent In particular, the closure is the largest subset of the domain that is guaranteed to be a subset of all hypotheses in the version space $\hyH_S$.
    In particular, since the target concept $\ell$ is always in the version space, $\Cl_{\cH}\!\inparen{S}$ is a subset of $\ell$.
  
\section{Characterizing Conditions}\label{sec:characterizing-conditions}
In this section, we introduce the conditions used to characterize proper positive-only learning.

All of the conditions are stated using the closure operator from \cref{sec:preliminaries}.
Recall that if $S$ is a finite set of observed positive examples and $S\subseteq \ell\in \cH$, then $\Cl_{\cH}\!\inparen{S} \subseteq \ell.$
Thus every point in $\Cl_{\cH}\!\inparen{S}$ is forced to be positive by the sample and the realizability assumption.
The main challenge is what happens outside the closure.
A point $x\notin \Cl_{\cH}\!\inparen{S}$ is not forced to be positive by the sample, and a proper learner must still output a member of $\cH$ without labeling too much of this outside region as positive.
The conditions below capture different ways in which hypotheses from $\cH$ can control this outside region.

\paragraph{The improper benchmark.}
Before stating the conditions, it is useful to recall the characterization of \textit{improper} positive-only learning.
Let $\cH_{\cap}$ denote the finite-intersection closure of $\cH$:
\[
    \cH_{\cap}
    \coloneqq
    \inbrace{
        \bigcap\nolimits_{h\in \mathcal{A}} h
        :
        \mathcal{A}\subseteq \cH \text{ is finite}
    }\, . 
\]
(Here, the empty intersection is interpreted as $\cX$.) For instance, if $\cH$ is the class of halfspaces in $\R^d$, then $\cH_{\cap}$ contains all polyhedra obtained as intersections of finitely many halfspaces.
\begin{theorem}[Improper positive-only learning; \citep{graus1989lower}; see also \citep{kearns1994introduction,lee2026smoothed}]
\label{thm:improper-positive-only}
    A concept class $\cH$ is improperly learnable from positive-only samples if and only if
    $\vc{}\!\inparen{\cH_{\cap}} < \infty.$
\end{theorem}
\noindent The learning algorithm is simple: given positive samples $S$, it outputs $\Cl_{\cH}\!\inparen{S}$.
Since this closure lies in $\cH_{\cap}$ but generally not in $\cH$, the learning algorithm is improper.
A proper learner cannot do this, and the conditions below ask, in increasingly weaker forms, how well concepts in $\cH$ can mimic this.

The strongest condition asks that the closure itself is always proper.
\vspace{-1mm}
\begin{definition}[Exact Exterior Separation]\label{def:exact-exterior-separation}\label{def:ees}
    We say that $\cH$ satisfies \emph{exact exterior separation} if for every finite nonempty realizable set $S \subseteq \cX$,
    $\Cl_{\cH}\!\inparen{S} \in \cH .$%
\end{definition}
\noindent Under exact exterior separation, the proper learner can directly output $\Cl_{\cH}\!\inparen{S}$, recovering the improper strategy while being proper.
This condition, however, turns out to be too strong (there are classes which violate this property while still being properly learnable from positive examples).
Hence, we consider the following weaker version.
\begin{definition}[Uniform Exterior Separability]\label{def:uniform-exterior-separability}\label{def:ues}
    We say that $\cH$ satisfies \emph{uniform exterior separability} if for every $\eta>0$ there exists an integer $M\!\inparen{\eta}$ such that for every finite nonempty realizable set $S\subseteq \cX$, there are hypotheses
    $h_{S,1},\dots,h_{S,M\!\inparen{\eta}} \in \cH_S$
    satisfying 
    \[
        \sup\nolimits_{x\notin \Cl_{\cH}\!\inparen{S}}~
        \frac{1}{M\!\inparen{\eta}}
        \abs{\inbrace{i\in\inbrace{1,\dots,M\!\inparen{\eta}}:x\in h_{S,i}}}
        \leq \eta \,.
    \]
\end{definition}
\noindent Here, instead of a single consistent hypothesis avoiding every exterior point, we ask only for a bounded \emph{list} of consistent hypotheses, whose individual members may make false positives, but which on average covers any single exterior point only ``sparingly.''
The key point is that $M\!\inparen{\eta}$ depends only on $\eta$, and not on $S$.
Equivalently: a hypothesis drawn uniformly from the list labels any fixed exterior point as positive with probability at most $\eta$.
Finally, observe that exact exterior separation (\cref{def:ees}) is the special case $M\!\inparen{\eta}=1$ with $h_{S,1}=\Cl_{\cH}\!\inparen{S}$.

The next condition replaces the uniform distribution over a bounded list with an arbitrary distribution, which is a bit more natural for randomized learners. %
\begin{definition}[Distributional Exterior Separability]\label{def:distributional-exterior-separability}\label{def:des}
    We say that $\cH$ satisfies \emph{distributional exterior separability} if for every finite nonempty realizable set $S\subseteq \cX$ and every $\eta>0$, there exists a probability distribution $\mu_{S,\eta}$ supported on $\cH_S$ such that
    \[
        \sup\nolimits_{x\notin \Cl_{\cH}\!\inparen{S}}
        \Pr\nolimits_{h\sim \mu_{S,\eta}}\!\inparen{x\in h}
        \leq \eta \,.
    \]
    \vspace{-5mm}
\end{definition}
\noindent In the next section, we use these conditions to give our main characterization. 
Before that, we establish the following relationships between these conditions to build some intuition; the proof is deferred to \cref{sec:hierarchy}.
\begin{restatable}[Relations between the exterior conditions]{proposition}{exteriorSeparationHierarchy}
\label{prop:hierarchy}
For every concept class $\cH$,
\[
    \begin{aligned}
    \textnormal{exact exterior separation}
    &~~\Longrightarrow~~
    \textnormal{uniform exterior separability}
    \\
    &~~\Longrightarrow~~
    \textnormal{distributional exterior separability}\,.
    \end{aligned}
\]
\noindent Moreover, the first implication is strict.
\end{restatable}
\noindent The two implications follow from the definition of the conditions.
Strictness of the first implication is important to refute a conjecture of \cite{natarajan1987learning} (see \cref{sec:related-works}).  
The second implication may not be strict in general. However, interestingly, using tools from the study of sample compression schemes \cite{moran2016sample}, we are able to prove that distributional exterior separability implies uniform exterior separability for classes $\cH$ with finite VC dimension, and this result turns out to be crucial for our characterization.
The proof of this result appears in \cref{subsec:des-to-ues}.
\begin{restatable}[From distributional to uniform exterior separability]{theorem}{distributionalImpliesUniform}
\label{thm:distributional-implies-uniform}
    Suppose that $\vc{}\!\inparen{\cH}<\infty$.
    If $\cH$ satisfies distributional exterior separability, then $\cH$ satisfies uniform exterior separability.
\end{restatable}
\vspace{-7mm}

\section{Our Results}
    \label{sec:results}
    In this section, we state our main results. 
    Our first result settles the question posed in \cref{sec:introduction}.
    \begin{restatable}[Main characterization]{theorem}{RandChar} \label{thm:random-characterization}\label{thm:main}
    The following are equivalent:
        \begin{enumerate}[itemsep=-2pt,leftmargin=15pt,topsep=0em]
            \item $\cH$ is properly learnable from positive-only samples;
            \item $\cH$ satisfies uniform exterior-separability (\cref{def:ues}) and $\vc{}\!\inparen{\cH}<\infty$.
        \end{enumerate}
    \end{restatable}
    \noindent Thus, uniform exterior-separability plays the role for proper positive-only learning that finite VC dimension plays for ordinary PAC learning.
    This characterization is strictly stronger than that of improper positive-only learning (\cref{thm:separation-improper}; see \cref{sec:additional-results:separation-improper} for a proof).
    \highlight{Thus, unlike ordinary PAC learning, where requiring the learner to be proper does not change which classes are learnable, characterizations of proper and improper learning differ in the positive-only model, \eg{}, \citep{shalev2014understanding}.}
    The proof of \cref{thm:random-characterization} appears in \cref{sec:random-characterization}.
    
    The learner in \cref{thm:main} is randomized, raising the question:  \textit{how much randomness is needed for learning?}
    Randomness is a fundamental algorithmic resource (alongside time and memory) and is widely studied in theoretical computer science \citep{arora2009computational}.
    It has also been extensively studied in learning theory both as a resource to minimize \citep{hopkins2025the,larsen2024derandomizing,chase2023stability,filmus2023optimal} and through models in which the learner or environment must be deterministic \citep{maass1991online,oliveira2018pseudo,sitharam1997derandomized}.

    Our next result shows that a very small number of random bits suffice for learning.
    To state our result, we need to recall the definition of the dual concept class $\hyH^\star$: for a binary class $\hyH \subseteq \zo{}^{\cX}$, its dual class is $\hyH^\star \coloneqq \{h \mapsto h(x) : x \in \cX\} \subseteq \zo{}^{\hyH}.$
    (Here, each point $x \in \cX$ defines a binary function on hypotheses by evaluating them at $x$.)
    The dual class $\hyH^*$ is useful because $\VC(\hyH^*)$ controls how many sample points are needed to distinguish hypotheses in $\hyH$; this parameter appears throughout learning theory, perhaps most famously in bounds on the size of sample-compression schemes \citep{moran2016sample}.
    \begin{restatable}[Few random bits suffice]{theorem}{FewBits}
    \label{thm:few-random-bits-suffice}
    Consider a concept class $\cH$ and its dual $\cH^\star$.
    Let $d=\vc{}\!\inparen{\cH}$ and $d^\star=\vc{}\!\inparen{\cH^\star}$.
    The following are equivalent:
    \begin{enumerate}[itemsep=-1pt,leftmargin=15pt,topsep=0pt]
        \item $\cH$ is properly PAC learnable from positive-only examples by a randomized learner.
    
        \item $\cH$ is properly PAC learnable from positive-only examples by a learner using at most
        \vspace{-1mm}
        \[
            r\inparen{\eps,\delta}
            =
            \wt{O}\!\inparen{
                \log\inparen{\nfrac{d^\star}{\eps\delta}}
            }~\text{random bits}
            \quadand
            m\inparen{\eps,\delta}
            =
            \wt{O}\!\inparen{
                (\nfrac{1}{\eps})\, \inparen{d+\log\nfrac1\delta}
            }
            \text{~positive samples}\,.
        \]
        \vspace{-7mm}
    \end{enumerate}
    \noindent In particular, by Assouad's bound \citep{Assouad1983},
    $d^\star<2^{d+1}$, so 
    $
        r\inparen{\eps,\delta} = \wt{O}\!\inparen{ d + \log\inparen{\nfrac{1}{\eps\delta} }}.
    $   
    \end{restatable}
    \vspace{-5mm}

    \noindent For classes with VC dimension $O(1)$, the number of random bits is only $\wt{O}\!\inparen{\log\nfrac{1}{\eps\delta}}$.
    This is much smaller than the number of samples: even in ordinary PAC learning with positive and negative examples, the sample complexity is $\Omega(\nfrac{1}{\eps})$, which is exponentially larger.
    For additional details and proof of \cref{thm:few-random-bits-suffice} we refer the reader to \cref{sec:few-random-bits-suffice}.
    Given that so few random bits suffice, one might hope that randomness can be removed entirely.
    Our next result shows that this is not possible: some classes are properly learnable by a randomized \mbox{learner but by no deterministic learner.}

    \begin{restatable}[Randomness is necessary]{theorem}{RandNecessary}
    \label{thm:randomness-necessary}
        There exists a concept class $\cH$ that is properly positive-only learnable by a randomized learner, but not by any deterministic learner.
    \end{restatable}
    \noindent Thus, the power of randomized and deterministic learners is different in the proper positive-only learning model.
    \highlight{This is another contrast with ordinary PAC learning where characterizations of deterministic and randomized proper learning are equivalent, \eg{}, \citep{BlumerEhrenfeuchtHausslerWarmuth1989,ehrenfeucht1989general}.}  
    The proof of \cref{thm:randomness-necessary} appears in \cref{sec:randomness-necessary}. 

    The preceding results already give two separations between positive-only learning and ordinary PAC learning: proper and improper learning differ (\cref{thm:main}; also see \cref{thm:separation-improper}), and randomized and deterministic proper learning differ (\cref{thm:randomness-necessary}).
    We next state three further separations.

    \paragraph{Separation III: ERM is not a universal learner for positive-only learning.}
    In the realizable PAC model, if a class $\hyH$ is learnable, then it is learnable by \textit{every} empirical risk minimization (ERM) rule.
    An ERM rule is any rule that outputs any hypothesis from $\cH$ consistent with the observed samples. 
    Specifically, in the positive-only model, a learning rule is ERM if it outputs a hypothesis that contains all the observed positive examples. %
    Our next result shows that such rules are not universal learners for deterministic proper positive-only learning.
    Its proof appears in \cref{sec:ERM-separation}.
    \begin{restatable}[ERM is not a universal deterministic learner]{theorem}{SepERM} \label{thm:ERM-separation}
        There exists a concept class $\cH$ that is properly positive-only learnable by deterministic learners, but no deterministic proper ERM learner learns $\cH$ from positive-only examples.
        
    \end{restatable}
    \vspace{-6mm}

    \paragraph{Separation IV: Finite VC dimension does not imply non-uniform learnability.}
        So far, we have focused on ``uniform'' guarantees on the sample complexity, where the sample complexity is the same for all target concepts.
        Non-uniform PAC learning relaxes this and allows the required sample size to depend on the target concept.
        Consistency further relaxes this: it only requires the error to converge to 0 for each fixed distribution-target pair as the number of samples goes to $\infty$ without requiring any specific \textit{rate} for this convergence.
        (We refer the reader to \cref{sec:additional-preliminaries} for formal definitions.)
        Both non-uniform and consistency guarantees are widely-studied in learning theory \citep{stone1977consistent,benedek1988nonuniform, ben1993learning, bousquet2021theory,lu2024inductive} and below we study these notions for proper positive-only learning.
        We show that these notions remain distinct for positive-only learning, even for classes of finite VC dimension; see \cref{sec:nonuniform-separation} for the proof.
        \begin{restatable}[Separation of learnability, non-uniform learnability, and consistency]{theorem}{SepNonUnif} \label{thm:nonuniform-separation}
            There exist concept classes $\cH_1$ and $\cH_2$, each with VC dimension $1$, such that:
            \begin{enumerate}[itemsep=-1pt,leftmargin=15pt,topsep=0pt]
                \item $\cH_1$ is properly positive-only \ul{consistent}, but not \ul{non-uniformly} properly positive-only learnable.
                \item $\cH_2$ is \ul{non-uniformly} properly positive-only learnable, but not properly positive-only learnable.
            \end{enumerate}
        \end{restatable}
 
\smallskip
    \noindent\textbf{Stable proper positive-only learning.}
    Stability is a standard way to formalize the idea that small changes to the sample should not substantially change the output of the learner.
    We end with a characterization of stable proper positive-only learning.
    Concretely, we study the following notion of stability, which matches the notion of stable compression schemes introduced by \citet{bousquet2020proper}.
    \begin{definition}[Stability] \label{def:stable}
        Consider a proper learner $\cA$.
        The learner $\cA$ is said to be \emph{stable} if, for any multi-sets $S, S' \in \cX^{[*]}$, with probability 1 over any randomness in $\cA$, the following holds
        \[
            \text{if}\quad
            S\subseteq S' \text{ and } \Domain(S') \subseteq \cA(S) ,\quadtext{then}
            \cA(S')=  \cA(S)\,,
        \]
        where $\Domain(S')$ is the set of distinct elements in $S'$. 
    \end{definition}
    \noindent The characterization shows that stability brings us back to exact exterior separation, the strongest closure condition from \cref{def:ees}. 
    \begin{restatable}[Characterization for stable learners]{theorem}{StabChar}
    \label{thm:stable-characterization}
    The following are equivalent:
    \begin{enumerate}[itemsep=-2pt,leftmargin=15pt,topsep=0pt]
        \item $\cH$ is properly positive-only learnable by a stable 
        learner.
        \item $\cH$ satisfies exact exterior separation (\cref{def:ees}) and $\vc{}\!\inparen{\cH}<\infty$.
    \end{enumerate}
        
    \end{restatable}
    \noindent In fact, the necessity of exact exterior separation (\cref{def:ees}) and bounded VC dimension holds even under a weaker notion of stability. We refer the reader to \cref{sec:stable-characterization} for further details and the proof of \cref{thm:stable-characterization}.

    \vspace{-3mm}

\section{Technical Overview}
    In this section, we outline the proof of our main results.

    \smallskip

    \noindent\textbf{A Natural Approach to Proper Positive-Only Learning.}\quad
        A natural way to learn from positive-only examples is to output the closure of the observed positive samples $\sampleP$, namely $\Cl_\cH(\sampleP) \coloneqq \bigcap_{h\in \cH:\, \sampleP\subseteq h} h$.
        This is an ERM rule for positive-only learning and has the property that it makes no false positive errors (since the target $\ell$ itself contains $\sampleP$ and so is a part of the intersection).
        Its false negative errors are less straightforward but, if the output belongs to a finite VC dimension class, then it can be bounded using standard tools.
        The key issue is that it may not be \textit{proper} as the closure can fall outside $\cH$.
        If $\cH$ is intersection-closed, then this issue disappears, and combined with finite VC dimension, this closure rule becomes a proper positive-only learner.
        This sufficient condition was already known from the work of \citet{natarajan1987learning}, who additionally showed it is necessary if the learner is required to satisfy an additional constraint---make no false positive errors---and conjectured that it remains necessary even without this additional constraint.
        A consequence of our main characterization, \cref{thm:main}, is that this conjecture is \textit{false}.

    \paragraph{Refuting Natarajan's Conjecture.}
        To build toward the proof of \cref{thm:main}, we first describe how we refute Natarajan's conjecture.
        Consider the class
        \[
            \cH_{\mathrm{spr}}
            \coloneqq
            \inbrace{\inbrace{1},\inbrace{2}}
            \cup
            \inbrace{\inbrace{1,2,i}: i\geq 3}
        \]
        on the domain $\Nat$.
        This class has finite VC dimension but is not intersection-closed. 
        For instance, $\inbrace{1,2,3}\cap \inbrace{1,2,4} = \inbrace{1,2}\notin \cH_{\mathrm{spr}}$.
        So the closure rule cannot be applied properly: after seeing a sample $\sampleP$ consisting of only $1$s and $2$s, the closure $\inbrace{1,2}$ is the ``safe'' answer but it does not belong to $\cH_{\mathrm{spr}}$, and Natarajan's conjecture would, hence, predict that $\cH_{\mathrm{spr}}$ is not properly learnable.
        
        We refute this prediction.
        Fix $\eps,\delta\in(0,1)$, select $N=N(\eps,\delta)$ so that $1/N\lesssim \eps\delta$, 
        and for a sample $\sampleP$ of size $n$, let $o(\sampleP)$ be the number of copies of $1$ in $\sampleP$.
        Consider the proper randomized learner\footnote{This learner can be simplified by, \eg{}, removing some cases; we discuss this learner as we can readily derandomize it.}
        \[
            \cA(\sampleP) =
            \begin{cases}
                \inbrace{1,2,i} & \text{if some } i\geq 3 \text{ appears in } \sampleP\,, \\
                \inbrace{1} & \text{if } o(\sampleP) \geq n - \sqrt{n}\,, \\
                \inbrace{2} & \text{if } o(\sampleP) \leq \sqrt{n}\,, \\
                \inbrace{1,2,M} & \text{otherwise, where } M\sim\Unif(\inbrace{3,\dots,N+2})\,.
            \end{cases}
        \]
        \noindent The first three cases cover the regime where the sample essentially identifies the target.
        If the target is $\inbrace{1}$ or $\inbrace{2}$, the sample consists only of the corresponding point and the learner outputs the target.
        If the target is $\inbrace{1,2,i^\star}$ and $\cD_+(i^\star)$ is non-negligible, then $i^\star$ is observed once $n$ is large enough and case~1 again outputs the target.
        Finally, if one of $\inbrace{1,2}$ has tiny mass under $\cD_+$, the corresponding singleton case fires; dropping that low-mass point creates only small false-negative error.

        The interesting regime is when the target is $\ell=\inbrace{1,2,i^\star}$ with $\cD_+(i^\star)$ much smaller than $1/n$, so that $i^\star$ is not observed with high probability.
        In this regime, the sample contains only $1$s and $2$s (each having more than $\sqrt{n}$ copies): from this, the learner can infer that the target has the form $\inbrace{1,2,i}$ but has no information about the value of $i$.
        Here, the closure $\inbrace{1,2}$ is the natural ``safe'' choice which ensures zero false positive errors (and is exactly what \citet{natarajan1987learning}'s strategy would output), but it is \textit{improper.}
        Our learner replaces this improper closure by a random proper extension $\inbrace{1,2,M}$, where $M$ is uniformly chosen from a large set of values.
        The extension correctly covers $1$ and $2$, and the only point that could be misclassified as positive is the random choice $M$.
        For any fixed $k\notin\ell$, the chance the learner outputs $\inbrace{1,2,k}$ is at most $1/N$, so the expected false-positive mass is
        \[
            \mathbb{E}_M\insquare{\cD\inparen{\cA(\sampleP)\setminus\ell}}
            =
            \sum\nolimits_{k\notin\ell}\Pr\insquare{M=k}\cdot \cD(k)
            \leq
            \frac{1}{N}\sum\nolimits_{k\notin\ell}\cD(k)
            \lesssim
            \eps\delta\,.
        \]
        \noindent Markov's inequality then upgrades this to a high-probability guarantee. %
        To conclude, the key insight is to randomize over proper supersets of the closure $\inbrace{1,2}$ so that no exterior point receives large probability mass; this gives a proper positive-only learner for a finite-VC class that is not intersection-closed, refuting Natarajan's conjecture.

    \paragraph{Is Randomization Necessary for the Refutation?}
        At first glance, randomization seems essential here: to remain proper, the learner had to commit to some proper extension $\inbrace{1,2,k}$ in place of the improper closure $\inbrace{1,2}$, and randomizing over $k$ seems like the only way to avoid placing large probability on any specific exterior point.
        Surprisingly, the external coin flips are not needed---for any fixed distribution $\cD$, a deterministic learner can extract sufficient randomness from the sample itself.
        Concretely, we construct a deterministic variant of the above algorithm in which the last case outputs $\inbrace{1,2,o(\sampleP)}$, using the count of $1$s in the sample in place of the external $M$.
        When both $1$ and $2$ have non-negligible mass under $\cD_+$, $o(\sampleP)$ is sufficiently \textit{anti-concentrated} that no fixed exterior $k$ is output with large probability---the same role that the atom bound $\Pr\insquare{M=k}\leq 1/N$ played above.
        We give the formal argument in \cref{sec:ERM-separation}.
        
        This ability to de-randomize the learner is, however, a feature of $\cH_{\mathrm{spr}}$: in general, \cref{thm:randomness-necessary} exhibits classes that are learnable by a randomized proper positive-only learner but by no deterministic one.

\newcommand{\paragraphit}[1]{\smallskip \noindent\emph{#1}}

    \paragraph{Characterization of Proper Positive-Only Learning.}
        We now turn to our main characterization.
        We prove that $\cH$ is properly learnable from positive-only samples if and only if $\vc(\cH)<\infty$ and $\cH$ satisfies uniform exterior separability, or UES (\cref{def:ues}).

        Roughly speaking, UES is the structural property of $\cH$ that lets the randomized strategy used for $\cH_{\mathrm{spr}}$ above be carried out uniformly across all realizable samples.
        Recall that for a finite realizable set $S\subseteq\cX$, $\cH_S=\inbrace{h\in\cH:S\subseteq h}$ denotes the set of hypotheses in $\cH$ that contain $S$, and the closure $\Cl_\cH(S)=\bigcap_{h\in\cH_S}h$ consists of the points forced to be positive by $S$.
        UES asks that for every accuracy level $\eta$, there is a list of proper hypotheses $h_{S,1},\dots,h_{S,M(\eta)}\in\cH_S$, of size $M(\eta)$ depending only on $\eta$ (not on $S$), such that no exterior point $x\notin\Cl_\cH(S)$ is included in more than an $\eta$-fraction of the list.
        Equivalently, choosing a hypothesis uniformly from this list places probability at most $\eta$ on any fixed exterior point.
        This is precisely the property exploited for the class $\cH_{\mathrm{spr}}$ above: although the closure $\Cl_\cH(S)$ may itself be improper, UES lets us simulate it by randomizing over proper hypotheses with sufficiently spread-out false positives.

        \paragraphit{Sufficiency.}
        The sufficiency direction is conceptually straightforward and is easiest to prove using the weaker distributional condition DES (\cref{def:des}), which is implied by UES via the uniform distribution over the UES list. (DES is weaker because it permits an arbitrary distribution over $\cH_S$ rather than only ones supported on lists of bounded size.)
        Given a positive sample $\sampleP$, let $S=\Domain(\sampleP)$, choose $\eta\asymp\eps\delta$, and output $h\sim\mu_{S,\eta}$, where $\mu_{S,\eta}$ is the DES distribution supported on $\cH_S$.
        The output is proper by construction and consistent with $\sampleP$ since $\mu_{S,\eta}$ is supported on $\cH_S$.
        We divide the error into two parts:
        \begin{itemize}[leftmargin=15pt,itemsep=-1pt]
            \item \emph{False negatives.} Since $\vc(\cH)<\infty$, a standard uniform-convergence argument under $\cD_+$ ensures that, for $\abs{\sampleP}$ large enough, every $h\in\cH_S$ satisfies $\cD(\ell\setminus h)\leq\nfrac{\eps}{2}$ with high probability.
            \item \emph{False positives.} Since $\Cl_\cH(S)\subseteq\ell$, every negative point $x\notin\ell$ is exterior, and DES gives $\Pr_{h\sim\mu_{S,\eta}}\inparen{x\in h}\leq\eta$. Averaging over $x\sim\cD$ shows that the expected false-positive mass is at most $\eta$, and Markov's inequality converts this into a high-probability bound as before.
        \end{itemize}
        \noindent Taking $\eta$ on the order of $\eps\delta$ and union-bounding these two events yields the desired guarantee.

        \medskip
        \paragraphit{Necessity.}
        The necessity direction proceeds in two steps: we first show that finite VC dimension is necessary, and then we show that UES is necessary.
        The first step is inherited from standard PAC learning; the second is the key part of the necessity direction.

        \medskip
        \paragraphit{Necessity of finite VC dimension.}
        The necessity of finite VC dimension is not specific to positive-only learning.
        Indeed, if $\cH$ were properly learnable from positive-only samples, then it would also be learnable in the usual PAC model with both positive and negative examples and an improper learner: if the target has very small $\cD$-mass, the learner may simply output $\varnothing$, and otherwise a labeled sample contains enough positive examples to simulate the positive-only learner.
        The necessity follows as ordinary PAC learnability requires $\vc(\cH)<\infty$ \citep{shalev2014understanding}. %

        \medskip
        \paragraphit{Necessity of UES.}
        The main difficulty is proving the necessity of UES, which we do in two stages.
        We first show that (the weaker requirement of) DES is necessary.
        Then, we use the fact that $\VCD(\hyH)<\infty$ to strengthen this necessity proof to UES.

        \paragraphit{Step 1 (DES is necessary):}
        Fix a finite realizable set $S\subseteq\cX$ and a parameter $\eta>0$.
        Suppose $\cA$ is a positive-only learner with sample complexity $n=\mpos{}_\cH(\eps,\delta)$ for $\eps\ll 1/\abs{S}$ and $\delta=\eta/2$.
        Run $\cA$ on a sample $\sampleP\sim U_S^n$, where $U_S$ is the uniform distribution over $S$, and let $\nu$ denote the resulting distribution over outputs.
        We claim that $\nu$ is essentially the desired DES witness.

        We construct a single hard-instance:
        Fix any $x\notin\Cl_\cH(S)$.
        By definition of the closure, there is some $\ell_x\in\cH_S$ with $x\notin\ell_x$.
        Consider the data distribution $\cD_x$ that places half its mass on $x$ and the remaining half uniformly on $S$.
        Under target $\ell_x$, the positive conditional distribution is exactly $U_S$, so $\cA$ sees precisely the sample distribution used to define $\nu$.
        If $\cA$ outputs a hypothesis containing $x$, then $x$ is a false positive of mass $\nfrac{1}{2}$, violating the PAC guarantee.
        Hence $\nu\inparen{\inbrace{h:x\in h}}\leq\delta$, which is exactly \mbox{the DES bound at $x$}.
        We can apply this for each $x\notin \CLOS_\hyH(S)$.

        However, there is one caveat: DES requires a distribution supported on $\cH_S$, while $\nu$ may put mass outside $\cH_S$ (since, \eg{}, $\cA$ may not be an ERM).
        We can modify our argument to handle this: if $\cA$ outputs some $h\notin\cH_S$, then $h$ misses at least one $s\in S$, creating false-negative error at least $\nfrac{1}{2\abs{S}}>\eps$ under $\cD_x$, so $\nu(\cH\setminus\cH_S)\leq\delta$.
        Reassigning this excess mass to an arbitrary fixed $h_S\in\cH_S$ yields a distribution $\mu$ supported on $\cH_S$ \mbox{with $\Pr_{h\sim\mu}\inparen{x\in h}\leq 2\delta=\eta$ for every exterior $x$, proving DES.}%

        \paragraphit{Step 2 (From DES to UES):}
        Upgrading DES to UES is the most technical step of the characterization. %
        Fix $S$ and let $\mu$ be the DES distribution over $\cH_S$, with every exterior point having marginal at most $\eta/2$.
        For each exterior point $x$, define the range $R_x^S=\inbrace{h\in\cH_S:x\in h}\subseteq\cH_S$; DES requires exactly that each such range has $\mu$-mass at most $\eta/2$.

        The key observation is that the family $\inbrace{R_x^S:x\notin\Cl_\cH(S)}$ is a sub-family of the \textit{dual class} $\cH^\star$, and finite primal VC dimension implies finite dual VC dimension $d^\star=\vc(\cH^\star)$.
        We can therefore apply a uniform-convergence argument to this dual class: sampling
        $M(\eta)=O\!\inparen{d^\star\log\inparen{\nfrac{1}{\eta}}/\eta}$ hypotheses from $\mu$ ensures, with positive probability, that every range of $\mu$-mass at most $\eta/2$ is hit by at most an $\eta$-fraction of the sample.
        Any such realization yields hypotheses $h_{S,1},\dots,h_{S,M(\eta)}\in\cH_S$ satisfying the UES guarantee, and crucially $M(\eta)$ depends only on $\eta$ and $d^\star$, not on $S$.
        Combining the necessity of DES with this DES-to-UES upgrade completes the proof of the characterization. 

        \paragraphit{The role of the dual class here is, in our view, a non-trivial and surprising element of the proof.}
        To our knowledge, this kind of discretization argument has previously surfaced only in the study of sample compression schemes \citep{moran2016sample}, a setting that bears no obvious connection to positive-only learning.
        Yet it is exactly the tool needed to convert the distributional condition DES into the combinatorial condition UES -- a connection that, in our view, is far from obvious a priori.

\section{Conclusion and Future Work}
The problem of PAC learning from positive-only examples has been studied in many forms, dating back to
\citet{valiant1984deductive}.
While positive-only learning with improper learners is well understood and is even a textbook exercise, our understanding of proper learning from positive-only examples is much less well developed. Indeed, even a characterization of learnability was not known.
In this work, we revisit and settle this problem by providing a characterization of concept classes that are properly learnable from positive-only examples (\cref{thm:main}).

Surprisingly, this characterization (along with additional analysis) shows that the landscape of proper positive-only learning is surprisingly rich and, in several concrete ways, quite different from the landscape of the usual PAC learning model (\cref{thm:separation-improper,thm:randomness-necessary,thm:nonuniform-separation,thm:ERM-separation}).
These results were unexpected in the sense that \citet{natarajan1987learning}, who studied a much more stringent variant of positive-only learning (where the learner is not allowed to make \textit{any} false positives), conjectured that imposing this requirement of ``no false positives'' has no impact on learnability.
We refute this conjecture using our characterization (\cref{thm:separation-natarajan}; also see \cref{sec:related-works}). 

Our work leaves several directions for proper positive-only learning. 
First, while we characterize learnability, the exact random-bit complexity remains open.
Second, the characterization suggests analogous results for non-uniform learnability and consistency. %
Third, after identifying \emph{which} classes are learnable from positive-only examples, it is natural to ask for optimal statistical rates and efficient algorithms.

\newpage

\printbibliography
\newpage
\appendix
\section{Additional Preliminaries} \label{sec:additional-preliminaries}
In this section, we introduce additional notation and definitions that were omitted from the main body of the paper due to space constraints.

\paragraph{Notations.} %
     For a multi-set $S$, define its domain as the set of distinct elements appearing in $S$, denoted by
$\Domain(S)  \coloneqq  \{x : x \in S\}$.
Furthermore, for a family $\{(\cD_a, \ell_a) \mid a \in A\}$, denote the conditional distributions by
$\cD_{+,a}  \coloneqq  (\cD_a)_+$ and $\cD_{-,a}  \coloneqq  (\cD_a)_-$.
Also, for any finite set $A$, denote by $U_A$ the uniform distribution over $A$.

\paragraph{VC dimension and dual VC.} For a concept class $\cH$, and a finite subset of the domain $S \subseteq \cX$, we say that $S$ is \emph{shattered} by $\cH$ if for every $F \subseteq S$ there exists a concept $h \in \cH$ such that $h \cap S = F$. The \emph{VC dimension} of $\cH$, denoted by $\vc(\cH)$, is defined by the largest $d \in \Nat \cup 0$ such that there exists a finite set of instances $S \subseteq \cX$ of size $d$ that is shattered by $\cH$. If no such largest $d$ exists, we say $\vc(\cH) = \infty$.
The \emph{dual class} of $\cH$ is the concept class $\cH^\star$ defined below over the domain $\cH$
\[
\cH^{\star}
\coloneqq
\inbrace{R_x : x \in \cX}
\qquadwhere
R_x \coloneqq \inbrace{h \in \cH : x \in h}\,.
\]
\noindent Its VC dimension, $\vc{}\!\inparen{\cH^{\star}},$
is called the \emph{dual VC dimension} of $\cH$.
By a classical theorem of \citet{Assouad1983}, if $\vc{}\!\inparen{\cH}<\infty$, then the dual VC dimension is also finite. In fact, the following quantitative bound holds; see also the discussion in~\cite{ChaseChornomazHannekeMoranYehudayoff2024}.
\begin{theorem}[\cite{Assouad1983}]
    \label{thm:dualVC}
    For any concept class $\cH$, $\vc{}\!\inparen{\cH^{\star}} < 2^{\vc{}\!\inparen{\cH}+1}.$
\end{theorem}

\vspace{-2mm}
\paragraph{Non-uniform learning, and consistency.}  We also consider two weaker notions of proper positive-only learnability. The first is a non-uniform notion, in which the required sample size is allowed to depend on the target concept. The second is consistency, which requires the expected error to vanish in the limit.

We say that a concept class $\cH$ is non-uniformly properly positive-only learnable by a proper learner $\cA$ if there exists a function $\mpos{}_{\cH}: (0,1) \times (0,1) \times \cH \mapsto \Nat$ such that for every distribution $\cD$ over $\cX$, $\ell \in \cH$, $\eps, \delta \in (0, 1)$, and $n \geq \mpos{}_{\cH}(\eps,\delta, \ell)$, the following is satisfied
    \[
    \Pr_{\sampleP \sim \cD_+^n}\inparen{
    \err_{\cD}\inparen{\cA\inparen{\sampleP},\ell} \leq \eps
    } \geq 1-\delta\,.
    \]

\noindent We say that $\cH$ is properly positive-only consistent, if there exists a proper positive-only learner $\cA$ such that for every distribution $\cD$ over $\cX$ and $\ell \in \cH$,
\[
\lim_{n \rightarrow \infty} \cE{\sampleP \sim \cD_+^n}{
\err_{\cD}\inparen{\cA\inparen{\sampleP},\ell}} = 0\,.
\]

\section{Relationships Between Conditions}
In this section, we prove the relationships between the exterior-separation conditions introduced in
\cref{sec:characterizing-conditions} and a new notion we call \emph{finite exterior separability}.
We first establish the general hierarchy between exact, uniform, distributional, and finite exterior
separability.
We then use this hierarchy to prove separations between proper and improper positive-only learning,
and between proper positive-only learning and the one-sided-error model of
\citet{natarajan1987learning}.
Finally, we show that, under finite VC dimension, distributional exterior separability can be
strengthened to uniform exterior separability.

\subsection{General Hierarchy of Relationships: Extension of \cref{prop:hierarchy}} \label{sec:hierarchy}
In this section, we prove a slightly stronger version of \cref{prop:hierarchy}, which includes
finite exterior separability (defined below) as the weakest condition in the hierarchy.

\begin{definition}[Finite Exterior Separability]\label{def:fes}
    We say that $\cH$ satisfies \emph{finite exterior separability} if for every finite nonempty realizable set $S \subseteq \cX$ and every finite set $F \subseteq \cX \setminus \Cl_{\cH}\!\inparen{S}$, there exists $h \in \cH_S$ such that $h \cap F = \varnothing$.
\end{definition}

\begin{restatable}[Exterior-separation hierarchy]{proposition}{exteriorSeparationHierarchyAppendix}
\label{prop:hierarchy-appendix}
For every concept class $\cH$,
    \begin{align*}
        \textnormal{exact exterior separation}
        &~~\Longrightarrow~~
        \textnormal{uniform exterior separability}\\
        &~~\Longrightarrow~~
        \textnormal{distributional exterior separability}\\
        &~~\Longrightarrow~~
        \textnormal{finite exterior separability}\,.
    \end{align*}
    \noindent Moreover, the first and third implications are strict.
\end{restatable}

\begin{proof}
If $\cH$ satisfies exact exterior separation, then for every finite nonempty realizable $S$ we may take the single hypothesis $h_{S,1}=\Cl_{\cH}\!\inparen{S}$; this proves uniform exterior separability with $M\!\inparen{\eta}=1$ for every $\eta>0$.

If $\cH$ satisfies uniform exterior separability, fix $\eta>0$ and write $M \coloneqq  M\!\inparen{\eta}$ for the bound from Definition~\ref{def:ues}. For each finite nonempty realizable $S$, let $\mu_{S,\eta}$ be the uniform distribution on the corresponding family $h_{S,1},\dots,h_{S,M}$. Then for every $x \notin \Cl_{\cH}\!\inparen{S}$,
\[
\Pr_{h \sim \mu_{S,\eta}}\inparen{x \in h} = \frac{1}{M} \abs{\inbrace{i : x \in h_{S,i}}} \leq \eta\,,
\]
so distributional exterior separability holds.

Finally, assume distributional exterior separability, fix a finite nonempty realizable $S$, and let $F \subseteq \cX \setminus \Cl_{\cH}\!\inparen{S}$ be finite. Choose $\eta < 1/\abs{F}$ and let $\mu_{S,\eta}$ witness Definition~\ref{def:des}. Then
\[
\mathbb{E}_{h \sim \mu_{S,\eta}}\insquare{\abs{h \cap F}} = \sum_{x \in F} \Pr_{h \sim \mu_{S,\eta}}\inparen{x \in h} \leq \abs{F}\eta < 1\,.
\]
\noindent Hence some $h \in \cH_S$ satisfies $\abs{h \cap F}=0$, proving finite exterior separability.
It remains to prove the strictness of the first and third implications.
We divide this into the two propositions below.
\begin{proposition}\label{prop:ees-strict}
There is a concept class $\cH$ which satisfies uniform exterior separability (\cref{def:ues}) but does not satisfy exact exterior separation (\cref{def:ees}).
\end{proposition}

\begin{proof}
Consider the class
\[
\cH_{\mathrm{spr}} \coloneqq  \inbrace{\inbrace{1},\inbrace{2}} \cup \inbrace{\inbrace{1,2,i} : i \geq 3}
\]
on the domain $\N$. It is not exact exterior separating, since for $S=\inbrace{1,2}$ we have $\Cl_{\cH_{\mathrm{spr}}}\inparen{S}=\inbrace{1,2}\notin \cH_{\mathrm{spr}}$.

We claim that $\cH_{\mathrm{spr}}$ nevertheless satisfies uniform exterior separability. Fix $\eta>0$ and let $M \coloneqq  \lceil 1/\eta \rceil$. If $S=\inbrace{1,2}$, choose
\[
h_{S,j} \coloneqq  \inbrace{1,2,2+j}\,, \qquad j=1,\dots,M\,.
\]
\noindent Every point outside $\Cl_{\cH_{\mathrm{spr}}}\inparen{S}=\inbrace{1,2}$ belongs to at most one of these $M$ hypotheses, so its average frequency is at most $1/M \leq \eta$.

For every other finite nonempty realizable $S$, the closure already belongs to the class: for instance, $\Cl_{\cH_{\mathrm{spr}}}\inparen{\inbrace{1}}=\inbrace{1}$, $\Cl_{\cH_{\mathrm{spr}}}\inparen{\inbrace{2}}=\inbrace{2}$, and if $S$ contains some $i \geq 3$ together with either $1$ or $2$, then $\cH_S$ is a singleton and its unique element is the closure. In all such cases we simply repeat $\Cl_{\cH_{\mathrm{spr}}}\inparen{S}$ exactly $M$ times. Thus $\cH_{\mathrm{spr}}$ satisfies uniform exterior separability.
\end{proof}

\begin{proposition}\label{prop:des-strict}
There is a concept class $\cH$ which satisfies finite exterior separability (\cref{def:fes}) but does not satisfy distributional exterior separability (\cref{def:des}).
\end{proposition}

\begin{proof}
Consider the class
\[
    \cH_{\mathrm{tail}} \coloneqq  \inbrace{h_a : a \geq 3}\,,
    \qquadwhere
    h_a \coloneqq  \inbrace{1,2} \cup \inbrace{m \in \N : m \geq a}\,.
\]
\noindent We first check finite exterior separability. Let $S \subseteq \N$ be finite, nonempty, and realizable.
If $S$ contains some point $m \geq 3$, then with $a_0 \coloneqq  \min \inparen{S \cap \inbrace{3,4,\dots}}$ we have $\Cl_{\cH_{\mathrm{tail}}}\inparen{S}=h_{a_0} \in \cH_{\mathrm{tail}}$, so the property is immediate. If $S \subseteq \inbrace{1,2}$, then $\Cl_{\cH_{\mathrm{tail}}}\inparen{S}=\inbrace{1,2}$, and given any finite set $F \subseteq \N \setminus \inbrace{1,2}$, choosing $a>\max F$ yields $h_a \in \cH_S$ with $h_a \cap F=\varnothing$. Thus $\cH_{\mathrm{tail}}$ satisfies finite exterior separability.

We now show that distributional exterior separability fails. Fix $S=\inbrace{1,2}$ and let $\mu$ be any probability distribution on $\cH_S=\cH_{\mathrm{tail}}$. For each $N \geq 3$,
\[
    \Pr_{h \sim \mu}\inparen{N \in h} = \sum_{a=3}^N \mu\!\inparen{h_a}\,.
\]
\noindent As $N \to \infty$, the right-hand side increases to $1$. Hence
\[
    \sup_{N \notin \Cl_{\cH_{\mathrm{tail}}}\inparen{S}} \Pr_{h \sim \mu}\inparen{N \in h} = 1\,,
\]
so no distribution on $\cH_S$ can make the exterior marginals arbitrarily small. Therefore $\cH_{\mathrm{tail}}$ does not satisfy distributional exterior separability. \qedhere
\end{proof} 
\end{proof}

\subsection{Strict Separation between Proper Positive-Only Learning and Other Models}

    In this section, we use the hierarchy from the previous subsection to prove two strict separations: first, between proper and improper positive-only learning, and second, between proper positive-only learning and the proper one-sided-error model introduced by \citet{natarajan1987learning}.

        \label{sec:additional-results:separation-improper}
        \begin{theorem}
            \label{thm:separation-improper}
            There exists a concept class that is improperly positive-only learnable, but not properly positive-only learnable. 
        \end{theorem}
        \begin{proof}
             Let $\cX=\{0,1,2\}$ and $\cH=\inbrace{\{0,1\},\{0,2\}}$.
            Then $\cH_{\cap} \coloneqq \inbrace{\{0\},\{0,1\},\{0,2\}}$.
            Note that $\vc(\cH_{\cap})=1$.
            Therefore, by \cref{thm:improper-positive-only}, $\cH$ is improperly positive-only learnable.
            However, let $S=\{0\}$ and $F=\{1,2\}$.
            Note that $F\cap \Cl_{\cH}(S)=F\cap S=\varnothing$, but there is no $h\in\cH$ such that
            $S\subseteq h$ and $h\cap F=\varnothing$.
            Thus, $\cH$ does not satisfy finite exterior separability (\cref{def:fes}), and, by \cref{prop:hierarchy-appendix}, it does not satisfy uniform exterior separability either (\cref{def:ues}). In \cref{thm:main}, we showed that any class that does not satisfy uniform exterior
            separability is not properly positive-only learnable.
        \end{proof}

        \label{sec:additional-results:separation-natarajan}
        \begin{theorem}
            \label{thm:separation-natarajan}
            There exists a concept class $\cH$ that is properly learnable from positive-only examples, but it is not properly learnable with one-sided error (the model proposed by \cite{natarajan1987learning}).
        \end{theorem}
        \vspace{-7mm}
        \begin{proof}[Proof of \cref{thm:separation-natarajan}]
            In \cref{prop:ees-strict}, we introduced a concept class $\cH_{\mathrm{spr}}$ with $\vc(\cH_{\mathrm{spr}}) = 1$ that satisfies uniform exterior separability (\cref{def:ues}), but does not satisfy exact exterior separation (\cref{def:ees}). From \cref{thm:random-characterization}, this implies that $\cH_{\mathrm{spr}}$ is properly learnable from positive-only examples. However, \citet{natarajan1987learning} shows that exact exterior separation is necessary for properly learning with one-sided error. 
        \end{proof}

\subsection{Proof of \cref{thm:distributional-implies-uniform}}\label{subsec:des-to-ues}

In this section, we prove \cref{thm:distributional-implies-uniform}.
The main idea is that finite dual VC dimension lets us discretize any distribution witnessing
distributional exterior separability into a uniform distribution over finitely many hypotheses, while
preserving the exterior-marginal bounds up to constant factors.

\distributionalImpliesUniform*
\noindent We also show that if $d^\star \coloneqq \vc{}\!\inparen{\cH^\star}$, then, in
\Cref{def:ues}, for any $\eta \in (0,1)$, one may take
\[
    M\!\inparen{\eta}
    =
    O\!\inparen{\frac{d^\star \log\!\inparen{1/\eta}}{\eta}}\,.
\]
\noindent We will use the following corollary of standard second-order uniform convergence, which has in the past been used in establishing the existence of sample-compression schemes \citep{moran2016sample}.
\begin{corollary}[Corollary of second-order uniform convergence for finite-VC classes; cf.~\cite{Boucheron_Bousquet_Lugosi_2005}]
\label{thm:eps-approximation}
    For every concept class $\mathcal{H}$ over a domain $\Omega$ with $\operatorname{VCdim}\!\inparen{\mathcal{H}} \leq d < \infty$, every probability distribution $\nu$ on $\Omega$, and each $\eps \in (0,1)$, there is a multiset $z_1,\dots,z_M \in \Omega$, with $M = O\!\inparen{{d\,(\log{\nfrac1\eps})}/{\eps}},$
    such that every $A \in \mathcal{H}$ satisfies
    \[
    \text{if}\qquad
    \nu(A)\leq \frac{\eps}{2}
    \qquadtext{then}
    \frac{1}{M}\abs{\inbrace{i \in \inbrace{1,\dots,M} : z_i \in A}} \leq \eps\,.
    \]
\end{corollary}
\begin{proof}[Proof of \cref{thm:eps-approximation}]
    This is an immediate implication of the following (standard) second-order uniform convergence result; cf.~Section~5.1.2 in \cite{Boucheron_Bousquet_Lugosi_2005}.
    \begin{theorem} \label{thm:uniform-convergence}
        There exists a universal constant $C>0$ such that for every concept class $\mathcal{H}$ over a domain $\Omega$ with $\operatorname{VCdim}\!\inparen{\mathcal{H}} \leq d < \infty$ the following holds.
    For every $\eps,\delta \in (0,1)$, if $Z_1,\dots,Z_m \sim \nu$ are i.i.d.\ and $m \geq C\inparen{d \log{\nfrac1\eps} + \log{\nfrac1\delta}}/\eps$, then with probability at least $1-\delta$,
    every $A \in \mathcal{H}$ satisfies
    \begin{align*}
        \nu\!\inparen{A \triangle A^\star}
        &\leq
        \wh{\nu}_m\!\inparen{A \triangle A^\star}
        +
        \sqrt{\wh{\nu}_m\!\inparen{A \triangle A^\star}\, \eps}
        +
        \eps\,,\\
        \wh{\nu}_m\!\inparen{A \triangle A^\star}
        &\leq
        \nu\!\inparen{A \triangle A^\star}
        +
        \sqrt{\nu\!\inparen{A \triangle A^\star}\, \eps}
        +
        \eps\,,
    \end{align*}
    where $\wh{\nu}_m(\cdot)$ is the empirical distribution: $\wh{\nu}_m(B)
    \coloneqq
    \frac{1}{m}\abs{\inbrace{i \in \inbrace{1,\dots,m} : Z_i \in B}}$ for each $B\subseteq\Omega$.
    \end{theorem}
    \noindent Indeed, let $\gamma \coloneqq \nfrac{\eps}{16}$.
Apply the standard second-order uniform-convergence theorem with reference set
$A^\star = \varnothing$ and confidence parameter $\delta = \nfrac12$.
Then for
\[
m = O\!\inparen{\frac{d\log{\nfrac1\gamma}}{\gamma}}
=
O\!\inparen{\frac{d\log{\nfrac1\eps}}{\eps}}\,,
\]
there exists a realization $z_1,\dots,z_m\in\Omega$ such that every $A\in\mathcal H$ satisfies
\[
\frac{1}{m}\abs{\inbrace{i \in \inbrace{1,\dots,m} : z_i \in A}}
\leq
\nu\inparen{A}+\sqrt{\nu\inparen{A}\gamma}+\gamma\,.
\]
\noindent If now $\nu\inparen{A}\leq \nfrac{\eps}{2}$, then
\[
\frac{1}{m}\abs{\inbrace{i \in \inbrace{1,\dots,m} : z_i \in A}}
\leq
\frac{\eps}{2}
+
\sqrt{\frac{\eps}{2}\cdot \frac{\eps}{16}}
+
\frac{\eps}{16}
\leq
\eps\,.
\]
\noindent Setting $M \coloneqq m$ proves the claim.
\end{proof}
\noindent Now we are ready to prove \cref{thm:distributional-implies-uniform}.
\begin{proof}[Proof of \cref{thm:distributional-implies-uniform}]
Recall that $d^\star$ is finite by the duality theorem (\cref{thm:dualVC}). Fix $0<\eta<1$ and a finite nonempty realizable set $S \subseteq \cX$.
By distributional exterior separability (\cref{def:des}), applied with parameter $\nfrac\eta2$, there exists a probability distribution $\mu$ supported on $\cH_S$ such that
\[
\sup_{x \notin \Cl_{\cH}\!\inparen{S}} \Pr_{h \sim \mu}\inparen{x \in h} \leq \frac{\eta}{2}\,.
\]
\noindent For each exterior point $x \notin \Cl_{\cH}\!\inparen{S}$, define
\[
R_x^S \coloneqq \inbrace{h \in \cH_S : x \in h}\,,
\]
and let
\[
\cR_S \coloneqq \inbrace{R_x^S : x \notin \Cl_{\cH}\!\inparen{S}}\,.
\]
\noindent Observe that $\cR_S$ is obtained from $\cH^\star$ by restricting the domain from $\cH$ to $\cH_S$ and then discarding the concepts only containing non-exterior points. 
Therefore, it follows that
\[
\vc{}\!\inparen{\cR_S} \leq \vc{}\!\inparen{\cH^\star} = d^\star\,.
\]
\noindent Since every set in $\cR_S$ has $\mu$-measure at most $\eta/2$, this is exactly the regime in which the second-order form of uniform convergence is useful.
We use \Cref{thm:eps-approximation} to the concept class $\cR_S$ over the domain $\cH_S$, under the probability measure $\mu$.
Since $\vc{}\!\inparen{\cR_S}\leq d^\star$, this yields a multiset
\[
h_{S,1},\dots,h_{S,M\!\inparen{\eta}} \in \cH_S
\qquad \text{with} \qquad
M\!\inparen{\eta} = O\!\inparen{\frac{d^\star \log{\nfrac1\eta}}{\eta}}
\]
such that every $R \in \cR_S$ with $\mu\inparen{R}\leq \nfrac\eta2$ satisfies
\[
\frac{1}{M\!\inparen{\eta}}
\abs{\inbrace{i \in \inbrace{1,\dots,M\!\inparen{\eta}} : h_{S,i} \in R}}
\leq \eta\,.
\]
\noindent Now fix any $x \notin \Cl_{\cH}\!\inparen{S}$. Since $R_x^S \in \cR_S$ and $\mu\inparen{R_x^S}
=
\Pr_{h \sim \mu}\inparen{x \in h}
\leq \frac{\eta}{2},$
the above implies
\[
\frac{1}{M\!\inparen{\eta}}
\abs{\inbrace{i \in \inbrace{1,\dots,M\!\inparen{\eta}} : x \in h_{S,i}}}
=
\frac{1}{M\!\inparen{\eta}}
\abs{\inbrace{i : h_{S,i} \in R_x^S}}
\leq \eta\,.
\]
\noindent Since $x$ was arbitrary, the family $h_{S,1},\dots,h_{S,M\!\inparen{\eta}}$ witnesses UES at level $\eta$.
As the bound on $M\!\inparen{\eta}$ depends only on $\eta$ and $d^\star$, this proves UES (\Cref{def:ues}).
\end{proof}

\section{Proof of Main Results}\label{sec:apdx-main-results}

In this appendix, we prove 
\cref{thm:random-characterization,thm:few-random-bits-suffice,thm:randomness-necessary} from \cref{sec:results}. 

\subsection{Proof of \cref{thm:random-characterization}} \label{sec:random-characterization}
 In this section, we prove \cref{thm:random-characterization}, which we restate below.
\RandChar*
\begin{proof}

    Observe that in \cref{thm:distributional-implies-uniform} we have proved that under the assumption $\vc(\cH) < \infty,$ uniform exterior separability is equivalent to distributional exterior separability (\cref{def:des}).
    We divide the proof into two parts, one for each direction of the claim.
    
    \paragraph{Part 1 ($2 \Longrightarrow 1$):} 
    Define the learner $\cA$ for any multi-set $\sampleP \in \cX^{[*]}$ as 
    \[\cA(\sampleP) \sim \mu_{\sampleP, \exp(-|\sampleP|)}\,.\] 
    Fix $\eps,\delta \in (0,1)$, a distribution $\cD$ over $\cX$ and $\ell \in \cH$. Let $\sampleP \sim \cD_+^n$, for $n \in \Nat$, and write $d \coloneqq  \vc{}\!\inparen{\cH}$.

    By the standard realizable VC bound applied to the positive distribution $\cD_+$ (see, \eg{}, \cite{BlumerEhrenfeuchtHausslerWarmuth1989, shalev2014understanding}), there is a universal constant $C > 0$ such that if
    \[n \geq \frac{C \inparen{d \ln{\nfrac{1}{\eps}} + \ln{\nfrac{1}{\delta}}}}{\eps},\]
    then with probability at least $1-\nfrac{\delta}{2}$ every $h \in \cH$ such that $\Domain(\sampleP) \subseteq h$ satisfies
    \[
    \cD_+\!\inparen{\ell \setminus h} \leq \frac{\eps}{2}\,.
    \]
    \noindent Since $\cD\!\inparen{\ell \setminus h} = \cD(\ell) \cdot \cD_+\!\inparen{\ell \setminus h} \leq \cD_+\!\inparen{\ell \setminus h}$, the same event implies $\cD\!\inparen{\ell \setminus h} \leq \nfrac{\eps}{2}$ for every such $h$.

Observe that, for every $y \notin \Cl_{\cH}\!\inparen{\sampleP}$, and $n \geq \ln \inparen{\nfrac{\delta \eps}{4}}$
\[
\cE{\cA}{\mathds{1} \inbrace{y \in \cA (\sampleP)}} \leq \exp(-n) \leq \frac{\eps \delta}{4}\,.
\]
    \noindent Because $\Cl_{\cH}\!\inparen{\sampleP} \subseteq \ell$, every negative point lies outside the closure, and therefore by the DES property (\cref{def:des})
\[
\cE{\cA}{\cD\!\inparen{\cA (\sampleP) \setminus \ell} } \leq \frac{\eps \delta}{4}\,.
\]
\noindent Thus, 
\[
\cE{\cS_+ \sim \cD_+^{n}, \cA}{\cD\!\inparen{\cA (\sampleP) \setminus \ell} } \leq \frac{\eps \delta}{4}\,.
\]
\noindent Markov's inequality gives
\[
    \Pr_{\cS_+ \sim \cD_+^{n}, \cA} \inparen{\cD\!\inparen{\cA (\sampleP) \setminus \ell} > \sfrac{\eps}{2}}  \leq  \frac{\delta}{2}\,.
\]
\noindent Combining this with the bound on false negatives, we obtain
\[
\Pr_{\cS_+ \sim \cD_+^{n}, \cA}\inparen{\err_{\cD}\!\inparen{\cA (\sampleP),\ell} > \eps} \leq \delta\,.
\]

    \paragraph{Part 2 ($1 \Longrightarrow 2$):} Fix a finite nonempty realizable set $S \subseteq \cX$ and a parameter $\eta>0$. If $\cX\setminus \Cl_{\cH}\!\inparen{S}=\varnothing$, then any point mass on a concept in $\cH_S$ witnesses distributional exterior separability, so there is nothing to prove. Thus assume $\cX\setminus \Cl_{\cH}\!\inparen{S}\neq \varnothing$.

If $\eta \geq 1$, the conclusion is again trivial, so assume $\eta<1$. Choose and fix one concept $h_S \in \cH_S$. For each $x \notin \Cl_{\cH}\!\inparen{S}$, choose $\ell_x \in \cH_S$ such that $x \notin \ell_x$; this is possible by the definition of the closure.
Let
\[
k \coloneqq \abs{S}\,,
\qquad
\eps \coloneqq \frac{1}{4k}\,,
\qquad
\delta \coloneqq \frac{\eta}{2}\,,
\]
and let $n \coloneqq \mpos{}_{\cH}\!\inparen{\eps,\delta}$ be the sample complexity of a proper learner $\cA$ at these parameters. Let $U_S$ denote the uniform distribution on $S$, and let $\nu$ be the law of $\cA\inparen{\sampleP}$ when
\[
\sampleP \sim U_S^n\,.
\]
\noindent Fix any exterior point $x \notin \Cl_{\cH}\!\inparen{S}$, consider distribution $\cD_x$ defined as
\[
\cD_x(x)=\frac12
\qquadand 
\cD_x(s)=\frac{1}{2k}\ \text{ for each } s\in S\,,
\]
and zero elsewhere. Also let $\ell_x$ be any concept with $x \notin \ell_x$ and $S \subseteq \ell_x$. Observe that since $x \notin \Cl_{\cH}\!\inparen{S}$ such $\ell_x$ always exists. Also, $\cD_{+, S}  \coloneqq  U_S$. Hence the learner sees the same sample law $U_S^n$ for every such $x$, so its output law is always $\nu$.

Now, if $x \in \cA\inparen{\sampleP}$, then $x$ is a false positive under the target $\ell_x$, and therefore
\[
\err_{\cD_x}\inparen{\cA \inparen{\sampleP },\ell_x} \geq \cD_x(x)=\frac12>\eps\,.
\]
\noindent By the PAC guarantee,
\[
\nu\!\inparen{\inbrace{h:x\in h}} = \Pr_{\sampleP\sim U_S^n, \cA}\inparen{x\in \cA\inparen{\sampleP}} \leq \delta\,,
\]
where the probability is over $\sampleP\sim U_S^n$ and
the learner's internal randomness.
Likewise, if $\cA\inparen{\sampleP}\notin \cH_S$, then it misses some point $s\in S$, and since $s\in \ell_x$ and $\cD_x(s)=1/\inparen{2k}>\eps$, this again forces
\[
\err_{\cD_x}\inparen{\cA\inparen{\sampleP},\ell_x}>\eps\,.
\]
\noindent Hence
\[
\nu\!\inparen{\cH\setminus \cH_S}
=
\Pr\inparen{\cA\inparen{\sampleP}\notin \cH_S}
\leq
\delta\,,
\]
where again the probability is over $\sampleP\sim U_S^n$ and
over the learner's internal randomness.
Define a probability distribution $\mu$ on $\cH_S$ by moving all $\nu$-mass outside $\cH_S$ onto the single concept $h_S$:
\[
\mu(A)
\coloneqq
\nu\!\inparen{A\cap \cH_S}
+
\nu\!\inparen{\cH\setminus \cH_S}\mathds{1}\inbrace{h_S\in A}\,,
\qquad
A\subseteq \cH\,.
\]
\noindent Then $\mu$ is supported on $\cH_S$, and for every $x \notin \Cl_{\cH}\!\inparen{S}$,
\[
\Pr_{h\sim\mu}\inparen{x\in h}
\leq
\nu\!\inparen{\inbrace{h:x\in h}}
+
\nu\!\inparen{\cH\setminus \cH_S}
\leq
\delta+\delta
=
\eta\,.
\]
\noindent Thus $\mu$ witnesses distributional exterior separability for the pair $\inparen{S,\eta}$. Since $S$ and $\eta$ were arbitrary, $\cH$ satisfies distributional exterior separability.
\end{proof}

\subsection{Proof of \cref{thm:few-random-bits-suffice}} \label{sec:few-random-bits-suffice}
In this section, we prove \cref{thm:few-random-bits-suffice}, which we restate below.
\FewBits*

\begin{proof}[Proof of \cref{thm:few-random-bits-suffice}]
    Observe that the second item implies the first immediately since a learner using finitely many random bits is a randomized learner. 
    Below we prove the converse.

Assume that $\cH$ is properly PAC learnable from positive-only examples by a
randomized learner.
Then, due to \cref{thm:random-characterization}, 
    $\vc{}\!\inparen{\cH}<\infty$ and $\cH$
satisfies distributional exterior separability. 
Now, applying the quantitative version
of \cref{thm:distributional-implies-uniform} from \cref{subsec:des-to-ues}, implies that for every $\eta \in (0,1)$ and
every finite nonempty realizable set $S\subseteq \cX$, there are hypotheses
$h_{S,1},\dots,h_{S,M\!\inparen{\eta}} \in \cH_S$
such that
$M\!\inparen{\eta}
    =
    O\!\inparen{
        (\nfrac{\inparen{d^\star+1}}{\eta})\,
        \cdot\,\log{\nfrac1\eta}
    }$
and
\[
    \sup_{x\notin \Cl_{\cH}\!\inparen{S}}
    \frac{1}{M\!\inparen{\eta}}
    \abs{
        \inbrace{
            i\in\inbrace{1,\dots,M\!\inparen{\eta}}:
            x\in h_{S,i}
        }
    }
    \leq \eta\,.
\]
\noindent Fix $\eps,\delta \in (0,1)$ and set
\[
    \eta \coloneqq \frac{\eps\delta}{8}\,,
    \qquadand
    M \coloneqq M\!\inparen{\eta}\,.
\]
\noindent For every finite nonempty realizable set $S$, fix a list $h_{S,1},\dots,h_{S,M} \in \cH_S$
at scale $\eta$ (as above).

\paragraph{Learning Algorithm.}
Let $K$ be a power of two satisfying
$M \leq K < 2M .$
The learner uses $\log_2 K$ random bits to draw a uniform seed
$J \sim U_{\inbrace{1,\dots,K}} .$
It then converts this seed into an index
$I(J)
    \coloneqq
    1+\inparen{J-1 \bmod M}
    \in \inbrace{1,\dots,M}.$
Thus, for every $i\in\inbrace{1,\dots,M}$,
\[
    \Pr\inparen{I(J)=i}
    \leq
    \frac{2}{M}\,.
\]

\noindent Given a positive multiset $\sampleP\in\cX^{[*]}$, let
$T \coloneqq \Domain\!\inparen{\sampleP}$
be the set of distinct examples in $\sampleP$. 
If $T$ is empty or is not realizable,
the learner outputs an arbitrary fixed concept in $\cH$.
Otherwise, the learner outputs
$\cA\inparen{\sampleP,J}
    \coloneqq
    h_{T,I(J)} .$
    Observe that this learner is proper, since $h_{T,I(J)}\in\cH$, and it is consistent with the
observed positives, since $h_{T,I(J)}\in\cH_T$.

\paragraph{Upper Bound on False Negative Rate.}
Fix a distribution $\cD$ over $\cX$ and a target concept $\ell\in\cH$, and let
$\sampleP \sim \cD_+^n .$
Write $d=\vc{}\!\inparen{\cH}$. By the standard realizable VC bound applied to
the positive distribution $\cD_+$,
if $n
    \gtrsim
    (\nfrac{1}{\eps})\cdot (d\log{\nfrac1\eps} + \log{\nfrac1\delta}),$
then with probability at least $1-\nfrac{\delta}{2}$ over $\sampleP$, every
$h\in\cH$ consistent with $\sampleP$ satisfies
$
    \cD_+\!\inparen{\ell\setminus h}
    \leq
    \nfrac{\eps}{2}.$
Since
$\cD\!\inparen{\ell\setminus h}
    =
    \cD\!\inparen{\ell}\,
    \cD_+\!\inparen{\ell\setminus h}
    \leq
    \cD_+\!\inparen{\ell\setminus h},$
the same event implies
$\cD\!\inparen{\ell\setminus h}
    \leq
    \nfrac{\eps}{2}$
for every consistent $h\in\cH$.
The learner's output is always consistent with $\sampleP$. Hence, conditioned on this
event,
$\cD\!\inparen{\ell\setminus \cA\inparen{\sampleP,J}}
    \leq
    \nfrac{\eps}{2}.$

\paragraph{Upper Bound on False Positive Rate.}
Condition on the observed multiset $\sampleP$, and therefore on
$T=\Domain\!\inparen{\sampleP}$. 
Since $\sampleP$ is realizable, $T\subseteq \ell$, and hence $\Cl_{\cH}\!\inparen{T}
    \subseteq
    \ell .$
Thus every negative point $x\notin \ell$ satisfies
$x\notin \Cl_{\cH}\!\inparen{T}.$
For such an $x$, the index $I(J)$ and the UES list satisfy
\begin{align*}
    \Pr\nolimits_J\!\inparen{x\in \cA\inparen{\sampleP,J}}
    =
    \Pr\nolimits_J\!\inparen{x\in h_{T,I(J)}}        
    \leq
    \frac{2}{M}
    \abs{
        \inbrace{
            i\in\inbrace{1,\dots,M}:
            x\in h_{T,i}
        }
    }                                       
    \leq
    2\eta\,.
\end{align*} 
\noindent Therefore,
\begin{align*}
    \mathbb{E}_J~ \insquare{\cD(
            \cA({\sampleP},J)\setminus \ell
        )
        |
        {\sampleP}}
    =
    \int_{\cX\setminus \ell}
    \Pr\nolimits_J(x\in \cA(\sampleP,J))
    \,\d \cD(x)   
    \leq 2\eta \,.
\end{align*} 

\noindent Next, by Markov's inequality,
\[
    \Pr_J\!\inparen{
        \cD\!\inparen{\cA\inparen{\sampleP,J}\setminus \ell}
        >
        \frac{\eps}{2}
        \,\middle|\,
        \sampleP
    }
    \leq
    \frac{2\eta}{\eps/2}
    =
    \frac{4\eta}{\eps}
    =
    \frac{\delta}{2}\,.
\]
\noindent Finally, averaging over $\sampleP$ gives
\[
    \Pr_{\sampleP,J}\!\inparen{
        \cD\!\inparen{\cA\inparen{\sampleP,J}\setminus \ell}
        >
        \frac{\eps}{2}
    }
    \leq
    \frac{\delta}{2}\,.
\]

\paragraph{Conclusion.}
Thus, with probability at least $1-\nfrac{\delta}{2}$ over the sample, the
false-negative mass is at most $\nfrac{\eps}{2}$. Also, with probability at
least $1-\nfrac{\delta}{2}$ over the sample and the learner's random bits, the
false-positive mass is at most $\nfrac{\eps}{2}$. A union bound gives that $\Pr_{\sampleP,J}\!\inparen{
        \err_{\cD}\!\inparen{\cA\inparen{\sampleP,J},\ell}
        >
        \eps
    }
    \leq
    \delta .$
Thus, the learner properly PAC learns $\cH$ from positive-only examples with
sample complexity
\[
    m\inparen{\eps,\delta}
    =
    O\!\inparen{
        \frac{
            d\log{\nfrac1\eps}
            +
            \log{\nfrac1\delta}
        }{\eps}
    }
    =
    \wt{O}\!\inparen{
        \frac{d+\log{\nfrac1\delta}}{\eps}
    }\,.
\]

\noindent It remains to bound the number of random bits. Since
$M
    =
    O\!\inparen{
        (\nfrac{\sinparen{d^\star+1}}{\eta})\cdot (\log{\nfrac1\eta})
    }$ and $\eta=\nfrac{\eps\delta}{8},$ 
\[ 
    \log_2 K
    \leq
    1+\log_2 M         
    =
    \wt{O}\!\inparen{
        \log{\frac{d^\star+2}{\eps\delta}}
    }\,. 
\]
\noindent Finally, by \cref{thm:dualVC},
$d^\star < 2^{d+1}$, it follows that $r\inparen{\eps,\delta}
    =
    \wt{O}\!\inparen{
        d+\log{\nfrac{1}{\eps\delta}}
    } .$\qedhere{}
\end{proof}

\subsection{Proof of \cref{thm:randomness-necessary}} \label{sec:randomness-necessary}

   In this section, we prove \cref{thm:randomness-necessary}, which we restate below.

   \RandNecessary*
   \noindent 
   The proof proceeds through a simple obstruction to deterministic proper learning.
We first introduce a property called \emph{singleton closure}, and then show in
\cref{prop:sc-necessary} that every class learnable from positive examples by a deterministic
proper learner must satisfy this property.
The theorem then follows by constructing a class that satisfies the conditions of
\cref{thm:random-characterization}, and is therefore learnable from positive-only examples, but violates
singleton closure and hence is not learnable by a deterministic learner.

\begin{definition}
[Singleton closure]
\label{def:sc}
A concept class $\cH$ is said to satisfy \emph{singleton closure} if for each realizable point
$x \in \cX$ (equivalently, for each $x$ with at least one $h\in \cH$ containing $x$), one has
\[
\Cl_{\cH}\!\inparen{\inbrace{x}} \in \cH\,.
\]
\end{definition}

\begin{proposition}[Singleton closure is necessary]\label{prop:sc-necessary}
If $\cH$ is properly learnable from positive-only examples by a deterministic learner, then $\cH$ satisfies singleton closure (\cref{def:sc}).
\end{proposition}

\vspace{-5mm}

\begin{proof}
Let $\cA$ be a deterministic proper learner such that $\cH$ is properly learnable from
positive-only examples by $\cA$.
Suppose, toward a contradiction, that
$\Cl_{\cH}\!\inparen{\inbrace{x}} \notin \cH$ for some realizable point $x$, and fix any
$n \in \Nat$.

Let $x^n$ denote the multi-set consisting of $n$ copies of $x$, and let
$h=\cA(x^n)$.
We consider two cases.

\paragraph{Case 1:} $x \notin h$.
Let $\ell \in \cH$ be any concept such that $x \in \ell$, and define
$\cD=U_{\inbrace{x}}$.
Then $\cD_+=\cD$, and the only possible positive sample is $\sampleP=x^n$.
Hence
\(
    \err_{\cD}(h,\ell)=1.
\)

\paragraph{Case 2:} $x \in h$.
Since $h\in\cH$ and $x\in h$, we have
\(
    \Cl_{\cH}\!\inparen{\inbrace{x}} \subseteq h .
\)
Because $\Cl_{\cH}\!\inparen{\inbrace{x}} \notin \cH$, it follows that there exists
\(
    x' \in h \setminus \Cl_{\cH}\!\inparen{\inbrace{x}} .
\)
By the definition of $\Cl_{\cH}\!\inparen{\inbrace{x}}$, there exists a concept
$\ell \in \cH$ such that $x \in \ell$ but $x' \notin \ell$.
Define $\cD=U_{\inbrace{x,x'}}$.
Then $\cD_+=U_{\inbrace{x}}$, so again the only possible positive sample is
$\sampleP=x^n$.
Moreover, since $x'\in h\setminus \ell$,
\(
    \err_{\cD}(h,\ell) \geq \frac{1}{2}.
\)

Thus, for every $n$, there exists a distribution $\cD$ and a target concept $\ell\in\cH$
such that
\[
    \Pr_{\sampleP \sim \cD_+^n}
    \insquare{
        \err_{\cD}(\cA(\sampleP),\ell) \geq \frac{1}{2}
    }
    =
    1\,.
\]
\noindent This contradicts the assumption that $\cH$ is properly learnable from positive-only examples by $\cA$.
\end{proof}

\begin{proof}[Proof of \cref{thm:randomness-necessary}]
Consider concept class 
\[\cH=\inbrace{\{1, a\} \mid a \in \Nat \setminus \{1\} }, \]
over $\cX = \Nat$. Notice that $\CLOS_{\cH}(\{1\}) = \{1\} \notin \cH$. Thus, $\cH$ does not satisfy the singleton closure property. Therefore, it is not properly learnable from positive examples by a deterministic learner.

However, $\vc(\cH) = 1$, and similarly to the proof of \cref{prop:ees-strict}, it is easy to see that $\cH$ satisfies the uniform exterior separability property. Therefore, due to \cref{thm:random-characterization}, it is properly learnable from positive examples.\qedhere{}

\end{proof}

\section{Proof of Additional Separation Results} \label{sec:additional-separation-results}
    
    First, we prove the following two structural results:
    \begin{enumerate}
        \item First, we prove that exact exterior separation (\cref{def:ees}) is sufficient for deterministic proper learning of VC classes, and
        \item Then, we prove necessary and sufficient conditions for consistency over countable domains.
    \end{enumerate}
    We then prove the remaining results from \cref{sec:results}.
    
\subsection{Exact Exterior Separation is Sufficient for Deterministic Learning} \label{sec:ees-sufficient}

We begin with proving that exact exterior separability is sufficient for deterministic proper learning, which will be used throughout this appendix.

\begin{lemma} \label{lem:int-closed-pos-learnable}
Every concept class $\cH$ that satisfies the exact exterior separability property (def~\ref{def:ees}) and has $\vc(\cH) < \infty$ is properly positive-only learnable.
\end{lemma}
\begin{proof}
    Fix any $\eps, \delta \in (0, 1)$ and distribution $\cD$ over $\cX$ and $\ell \in \cH$. Let the learner be defined as $\mathcal{A}(S) \coloneqq \CLOS_{\cH}(S)$. Let $n \geq \frac{d \ln{\nfrac{1}{\eps}} + \ln{\nfrac{1}{\delta}}}{\eps}$ and $\sampleP \sim \cD_+^{n}$.
    
    Observe that $\CLOS_\cH(\sampleP)$ is an ERM for $\cH$. Thus, we can use standard realizable PAC bounds (see, \eg{}, \cite{BlumerEhrenfeuchtHausslerWarmuth1989, shalev2014understanding}), and with probability at least $1  - \delta$,
    \[\cD \inparen{\ell \setminus \CLOS_\cH(\sampleP)} \leq \cD_+ \inparen{\ell \setminus \CLOS_\cH(\sampleP) } \leq \eps\,.\]  
    Moreover, due to realizability $\Domain(\sampleP) \subseteq \ell$, and, hence, $\CLOS_\cH(\sampleP) \subseteq \ell$. Subsequently, $\cD(\CLOS_\cH(\sampleP) \setminus \ell) = 0$. Therefore, with probability $1 - \delta$, $\err_\cD(\CLOS_\cH(\sampleP), \ell) \leq \eps$.  \qedhere{}
\end{proof}

\subsection{Necessary and Sufficient Conditions for Consistency} \label{sec:consistency-characterization}

We now characterize consistency over countable domains in terms of finite exterior separability.
This condition will later be used to separate PAC learnability, non-uniform learnability, and
consistency.

\begin{theorem}[Conditions for Consistency]\label{thm:consistency-characterization}
         The following holds:
         \begin{enumerate}
             \item No concept class $\cH$ which does not satisfy finite exterior separability (\cref{def:fes}) is properly positive-only consistent. 
             \item Every concept class $\cH$ over countable domain that satisfies finite exterior separability (\cref{def:fes}) is properly positive-only consistent.
         \end{enumerate}
\end{theorem}

\begin{proof} We divide the proof into two parts corresponding to the two claims.
    \smallskip
    
    \noindent\textbf{Part 1 (consistency implies FES).}\quad Consider any $S, F$ that violate finite exterior separability in \cref{def:fes}.
    For any $x \in F$ define $\cD_x$ as
    \[
    \cD_x(x) = \frac{1}{2}, \qquad \forall x' \in S: \cD_x(x') = \frac{1}{2\abs{S}}\,,
    \]
    and zero elsewhere. Also, let $\ell_x$ be any concept containing $S$ that does not contain $x$; such a concept exists since $x \notin \CLOS_{\cH}(S)$. 
    
    Since $S,F$ violate finite exterior separability, every $h \in \cH$ satisfies at least one of the
    following two conditions: either (1) $x_1 \notin h$ for some $x_1 \in S$, or (2) $x_2 \in h$ for some
    $x_2 \in F$.
    In case (1), 
    \[
    \sum_{x \in F} \err_{\cD_x} (h, \ell_x) \geq \frac{\abs{F}}{2\abs{S}}\,.
    \]
    \noindent In case (2), 
    \[
    \sum_{x \in F} \err_{\cD_x} (h, \ell_x) \geq \err_{\cD_{x_2}}(h, \ell_{x_2}) \geq \frac{1}{2}\,.
    \]
    \noindent Therefore, for either case
    \[
    \frac{\sum_{x \in F} \err_{\cD_x} (h, \ell_x)}{\abs{F}} \geq \frac{1}{2 \max(\abs{S}, \abs{F})}\,.
    \]
    \noindent This implies that, for every proper positive-only learner $\cA$ and every $n \in \Nat$,
    \begin{equation} \label{eq:consistency-implies-FES-1}
        \frac{\sum_{x \in F} \cE{\sampleP \sim U_{S}^{n}, \cA}{ \err_{\cD_x} (\cA(\sampleP), \ell_x)}}{\abs{F}} \geq \frac{1}{2 \max(\abs{S}, \abs{F})}\,.
    \end{equation}

    \noindent Now for the sake of contradiction, assume that $\cH$ is properly positive-only consistent. Therefore, there exists a proper learner $\cA$ such that for every $x \in F$,
    \[
        \lim_{n\to\infty}
        \cE{\sampleP \sim U_S^n,\cA}{
            \err_{\cD_x}(\cA(\sampleP),\ell_x)
        }
        =
        0\,.
    \]
    \noindent Since $F$ is finite, this would imply
    \[
        \lim_{n\to\infty}
        \frac{
            \sum_{x \in F}
            \cE{\sampleP \sim U_S^n,\cA}{
                \err_{\cD_x}(\cA(\sampleP),\ell_x)
            }
        }{\abs{F}}
        =
        0\,,
    \]
    contradicting \cref{eq:consistency-implies-FES-1}. Therefore, $\cH$ is not properly positive-only consistent, which completes the proof of the first part.

    \paragraph{Part 2 (FES over countable domain implies consistency).} Fix an enumeration $\cX=\inbrace{x_1,x_2,\dots}$ and write $X_m=\inbrace{x_1,\dots,x_m}$. Define the learner $\cA$ as follows. Given a realizable sample $\sampleP$, choose $\cA(\sampleP)\in\cH_{\Domain(\sampleP)}$ such that $\cA(\sampleP)\cap\inparen{X_{|\sampleP|}\setminus \CLOS_{\cH}(\sampleP)}=\varnothing$. Such a concept exists by finite exterior separability, applied to
    $S=\Domain(\sampleP)$ and $F=X_{|\sampleP|}\setminus \CLOS_{\cH}(\sampleP)$.
    
    Consider any distribution $\cD$ over $\cX$ and any $\ell \in \cH$. Since $\cX$ is countable,
    there exists $n_{\cD}(\eps)\in\Nat$ such that
    \[
        \cD\inparen{\cX\setminus X_{n_{\cD}(\eps)}} \leq \frac{\eps}{2}\,.
    \]
    \noindent Now fix any $n \geq n_{\cD}(\eps)$ and any multi-set $\sampleP \in \cX^n$. By construction,
    $\cA(\sampleP)$ has no false positives on
    $X_{n_{\cD}(\eps)} \setminus \CLOS_{\cH}(\sampleP)$, and since
    $\CLOS_{\cH}(\sampleP)\subseteq \ell$, it has no false positives on $X_{n_{\cD}(\eps)}$.
    Therefore,
    \begin{equation} \label{eq:const-char-decompose}
        \begin{aligned}
        \err_{\cD}(\cA(\sampleP), \ell) & \leq \cD(\ell \cap X_{n_{\cD}(\eps)} \setminus \cA(\sampleP)) + \cD(\cX \setminus X_{n_{\cD}(\eps)}) \\
        & \leq \cD(\ell \cap X_{n_{\cD}(\eps)} \setminus \cA(\sampleP)) + \frac{\eps}{2}\,.
    \end{aligned}
    \end{equation}
    \noindent Let $\cH' = \inbrace{h  \cap X_{n_{\cD}(\eps)} \mid h \in \cH}$. By definition $\vc(\cH') \leq n_{\cD}(\eps)$. Applying the standard realizable PAC bound to
    $\cH'$ and distribution $\cD_+$, for $n \gtrsim \frac{1}{\eps}\cdot \inparen{n_{\cD}(\eps) \ln{\nfrac{1}{\eps}}}$,
    \[
    \cE{\sampleP \sim \cD_+^{n}}{\cD(\ell \cap X_{n_{\cD}(\eps)} \setminus \cA(\sampleP))} \leq \eps/2\,.
    \]
    \noindent Combining with \eqref{eq:const-char-decompose} this implies $\cE{\sampleP \sim \cD_+^{n}}{\err_{\cD}(\cA(\sampleP), \ell)} \leq \eps.$
    Therefore,
    \[
    \lim_{n \to \infty} \cE{\sampleP \sim \cD_+^{n}}{\err_{\cD}(\cA(\sampleP), \ell)} = 0\,.\qedhere{}
    \]
\end{proof}

\noindent Recall that in \cref{thm:separation-improper} we presented a concept class that is improperly positive-only learnable, but does not satisfy finite exterior separability. By \cref{thm:consistency-characterization}, we immediately obtain the following corollary.

\begin{corollary}
    \label{cor:separation-improper-consistency}
            There exists a concept class that is improperly positive-only learnable, but not properly positive-only consistent. 
\end{corollary}

\subsection{Proof of \cref{thm:ERM-separation}} \label{sec:ERM-separation}
In this section, we prove \cref{thm:ERM-separation}, which we restate below.
\SepERM*
\noindent We begin with an anti-concentration lemma, \cref{lem:mode-n}, which allows us to bound the
largest probability of observing any fixed multi-set under $Q^n$ in terms of the largest atom of $Q$.
This lemma follows from a theorem of \citet{ushakov1986upper}.

\begin{theorem}[Theorem 3 of \citep{ushakov1986upper}] \label{thm:Kolmogorov-Rogozin}
Let $X_1, \dots, X_n$ be independent random variables over domain $\Omega$. For each $i \in [n]$, define
\[
p_i  \coloneqq  \sup_{\omega \in \Omega} \Pr[X_i = \omega]\,.
\]
\noindent Then there exists a universal constant $C > 1$ such that 
\[
\sup_{\omega \in \Omega} \Pr\!\left[\sum_{i=1}^n X_i = \omega \right]
\le
\sqrt{\frac{C}{\sum_{i=1}^n (1 - p_i)}}\,.
\]
\end{theorem}

\begin{lemma} \label{lem:mode-n}
    Let $\gamma \in (0,1)$, let $n > 0$, and let $Q$ be a distribution over a domain $\Omega$ satisfying $Q(\omega) \leq 1 - \gamma$. Then, for a universal constant $C > 1$ we have 
    \[
    \sup_{S \in \Omega^n} \Pr_{S' \sim Q^{n}}[S = S'] \leq \sqrt{\frac{C}{ n \gamma }}\,.
    \]
\end{lemma}
\begin{proof}[Proof of \cref{lem:mode-n}]
    Let $x_1,\dots,x_n$ be i.i.d.\ from $Q$, and let $N = (N_\omega)_{\omega \in \Omega}$
be the random histogram, where $N_\omega = \sum_{i=1}^n \mathds{1}\{x_i = \omega\}.$
Then $N$ completely determines the sampled multiset. So for any fixed multiset $S \in \Omega^n$, if $m(S)$ denotes its multiplicity vector, then
\[
\Pr_{S' \sim Q^n}[S' = S \text{ as multisets}]
=
\Pr[N = m(S)]\,.
\]

\noindent Now define, for each $i$, $Y_i = e_{x_i}$, where $e_\omega$ is the unit vector at coordinate $\omega$. Then
$
N = \sum_{i=1}^n Y_i\,.
$
So the probability of any fixed multiset is exactly a point mass of the sum of independent random vectors $Y_1,\dots,Y_n$.

For each $i$, the distribution of $Y_i$ has maximal atom
\[
\sup_{y} \Pr[Y_i = y]
=
\sup_{\omega \in \Omega} Q(\omega)
\leq 1 - \gamma\,.
\]
\noindent Hence
\[
1 - \sup_y \Pr[Y_i = y] \geq \gamma
\qquad \text{for every } i\,.
\]
\noindent Now apply Theorem~\ref{thm:Kolmogorov-Rogozin}, for independent random vectors $Y_1,\dots,Y_n$,
\[
\sup_x \Pr\!\left[\sum_{i=1}^n Y_i = x\right]
\le
\sqrt{\frac{C}{\sum_{i=1}^n \bigl(1 - \sup_y \Pr[Y_i = y]\bigr)}}\,.
\]
\noindent Therefore,
\[
\sup_x \Pr[N = x]
\le
\sqrt{\frac{C}{\sum_{i=1}^n \gamma}}
=
\sqrt{\frac{C}{n\gamma}}\,.\qedhere{}
\]
\end{proof}

\noindent The key idea of the proof is as follows. Consider the class $\cH_{\mathrm{spr}}$ used in the separation (see \cref{eq:ERM-separation-conceptclass}). $\cH_{\mathrm{spr}}$ is almost closed under intersections: the only positive samples
$\sampleP \in \cX^n$ for which $\CLOS_{\cH_{\mathrm{spr}}}(\sampleP) \notin \cH_{\mathrm{spr}}$ are those whose domain is exactly
$\{1,2\}$. Thus, all other samples can be handled in a straightforward way by outputting the corresponding
closure.

The main difficulty is therefore the case in which the learner observes a sample supported only on
$\{1,2\}$.
This event has non-negligible probability only when the target is of the form
$\{1,2,i^\star\}$, for some $i^\star \in \Nat\setminus\{1,2\}$, and the positive mass of
$i^\star$ is small.
In this regime, the learner uses the empirical imbalance between the number of $1$'s and $2$'s to
decide what to output: it returns $\{1\}$ if the number of $2$'s is very small, returns $\{2\}$ if
the number of $1$'s is very small, and otherwise returns a hypothesis of the form
$\{1,2,o(\sampleP)\}$, where $o(\sampleP)$ is determined by the observed multiplicity of $1$'s.

The analysis then splits according to the masses of $1$ and $2$ under $\cD_+$.
If one of these masses is small, the learner outputs the corresponding
singleton with high probability, and the resulting error is small because the omitted point has small marginal mass. On the other hand, if both $\cD_+(1)$ and $\cD_+(2)$ are at least on the order of $1/\sqrt n$, then \cref{lem:mode-n} implies that the probability of observing any fixed multi-set is small.
Consequently, the index $o(\sampleP)$ selected by the learner cannot have large marginal probability with significant chance, so the error of $\{1,2,o(\sampleP)\}$ is small. This proves that the specially designed deterministic learner succeeds.

The second part of the proof shows that no deterministic proper ERM learner can learn this class.

Now, we are ready to prove \cref{thm:ERM-separation}.
\begin{proof}[Proof of \cref{thm:ERM-separation}]
Recall the concept class defined in \cref{prop:ees-strict}
\begin{equation} \label{eq:ERM-separation-conceptclass}
   \cH_{\mathrm{spr}} \coloneqq  \inbrace{\inbrace{1},\inbrace{2}} \cup \inbrace{\inbrace{1,2,i} : i \geq 3}\,.
\end{equation}

\paragraph{Part 1 (Learnability by deterministic learners).} We first show that there exists a deterministic proper learner $\cA$ such that $\cH_{\mathrm{spr}}$ is proper positive-only learnable by $\cA$.

For every $n \geq 9$ and multi-set $\sampleP \in \cX^{n}$, let $o(\sampleP)$ denote the number of repetitions of $1$ in $\sampleP$. Consider the learning algorithm $\cA$ defined as follows
\begin{equation} 
    \cA(\sampleP) = \begin{cases}
        \{1, 2, i\} & \text{if } i \in \sampleP \setminus \{1, 2\} \\
        \{1\} & o(\sampleP) \geq n - \sqrt{n}\\
        \{2\} & o(\sampleP) \leq \sqrt{n}\\
        \{1, 2, o(\sampleP) \} & \text{o.w.}
    \end{cases}
\end{equation}
\noindent For samples of size less than $9$, define $\cA$ arbitrarily as any concept in $\cH_{\mathrm{spr}}$. Note that $\cA$ is proper and deterministic by definition. Consider any $\eps, \delta \in (0, 1)$, distribution $\cD$ over $\cX$, $\ell \in \cH_{\mathrm{spr}}$ and $n \geq \left(\nfrac{16C}{\eps \delta}\right)^4$ where $C$ is the universal constant in \cref{lem:mode-n}. We prove that
    \[
   \Pr_{\sampleP \sim \cD_+^{n}} \left[\err_{\cD}(\mathcal{A}\left(\sampleP\right), \ell) \geq \eps \right] \leq \delta\,.
    \]
    \noindent We prove the claim for 3 distinct cases of $\cD$ and $\ell$.

    \paragraph{Case 1:} $\ell=\{1\}$ or $\ell=\{2\}$. Then the learner receives a sample consisting only of the corresponding point. If $\ell=\{1\}$, then $o(\sampleP)=n$ and $\cA(\sampleP)=\{1\}$; if $\ell=\{2\}$, then $o(\sampleP)=0$ and $\cA(\sampleP)=\{2\}$. In both cases, the error is zero.

    \paragraph{Case 2:} $\ell = \{1, 2, i^\star\}$ and $\cD_+ (i^\star) > \nfrac{\eps}{4}$ for some $i^\star \in \Nat \setminus \{1, 2\}$. Since $n > \nfrac{4\ln(1 / \delta)}{\eps}$, we have
    \[
        \Pr[ i^\star \notin \sampleP] < (1 - \eps/4)^{n} \leq e^{- {n} \eps / 4} \leq \delta\,.
    \]
    \noindent Thus, with probability at least $1-\delta$, the sample contains $i^\star$, and $\cA$ outputs $\{1,2,i^\star\}$. On this event the error is zero.

    \paragraph{Case 3:} $\ell = \{1, 2, i^\star\}$ and $\cD_+ (i^\star) \leq \nfrac{\eps}{4}$ for some $i^\star \in \Nat \setminus \{1, 2\}$. Since $\ell = \{1, 2, i^*\}$, the false negative rate of $\cA$ is
    \begin{equation} \label{eq:thm:proper-intersect-closed-learnable-fn}
    \begin{aligned}
         \cD(\ell \setminus \mathcal{A}\left(\sampleP\right))  & = \cD(1) \, \mathds{1}[\mathcal{A}(\sampleP)(1) = 0] + \cD(2) \, \mathds{1}[\mathcal{A}(\sampleP)(2) = 0]  + \cD(i^*) \mathds{1}[\mathcal{A}(\sampleP)(i^\star) = 0] \\
        & \leq \cD_+(1) \, \mathds{1}[\mathcal{A}(\sampleP)(1) = 0] + \cD_+(2) \, \mathds{1}[\mathcal{A}(\sampleP)(2) = 0] + \eps / 4\,.
    \end{aligned}
    \end{equation}
    
    \noindent Then, we divide case 3 into two cases depending on values $\cD_+(1)$ and $\cD_+(2)$.

    \paragraph{Case 3.1:} $\cD_+(1)\leq \nfrac{1}{2 \sqrt {n}}$ or $\cD_+(2) \leq \nfrac{1}{2 \sqrt {n}}$. Without loss of generality assume $\cD_+(1)\leq \nfrac{1}{2 \sqrt {n}}$. Then, using multiplicative Chernoff bound, since ${n} > \left(\frac{6 \ln{\nfrac{1}{\delta}}}{\eps}\right)^2$ we have
    \[
    \Pr[o(\sampleP) > \sqrt{{n}}] \leq e^{-\sfrac{\sqrt{{n}}}{6}} \leq \delta\,.
    \]
    \noindent Thus, with probability at least $1 - \delta$, $\mathcal{A}(\sampleP)$ is either $\{2\}$ or $\{1, 2, i^\star\}$. In the latter case the error is zero. In the former case, \cref{eq:thm:proper-intersect-closed-learnable-fn} implies \[\err_\cD(\mathcal{A}(\sampleP), \ell) = \cD(\ell \setminus \mathcal{A}(\sampleP)) \leq \cD_+(1)+ \frac{\eps}{4} \leq \frac{1}{2 \sqrt{{n}}} + \frac{\eps}{4}.\] 
    Therefore, since ${n} > \left(\nfrac{2}{3 \eps}\right)^2$ we have $\err_\cD(\mathcal{A}(\sampleP), \ell) \leq \eps$.

    \paragraph{Case 3.2:} $\cD_+(1), \cD_+(2) > \nfrac{1}{2 \sqrt {n}}$. First suppose $\cD_+(1)\leq \nfrac{2}{\sqrt {n}}$.  Using ${n} \geq \left(\nfrac{8}{\eps}\right)^2$, we get
    \begin{equation} \label{eq:thm:proper-intersect-closed-learnable-bound1}
        \cD_+(1)\mathds{1}[\mathcal{A}(\sampleP)(1) = 0] \leq \frac{2}{\sqrt {n}} \leq \frac{\eps}{4}\,.
    \end{equation}
     \noindent Otherwise, $\cD_+(1)>\nfrac{2}{\sqrt n}$. Since ${n} \geq \left(\frac{4 \ln(4 / \delta)}{\eps}\right)^2$, another multiplicative Chernoff bound gives
    \[
   \Pr[o(\sampleP) \leq \sqrt{{n}}] \leq e^{-\sfrac{\sqrt{{n}}}{4}} \leq \frac{\delta}{4}\,.
    \]
    \noindent Thus, with probability at least $1 - \nfrac\delta4$ we have $\mathcal{A}(\sampleP)(1) = 1$, and subsequently 
    \[
    \cD_+(1)\mathds{1}[\mathcal{A}(\sampleP)(1) = 0] = 0\,.
    \] 
    \noindent Combining this with \eqref{eq:thm:proper-intersect-closed-learnable-bound1}, shows that, regardless of the value of $\cD_+(1)$, with probability $1-\nfrac\delta4$,
 \[\cD_+(1)\mathds{1}[\mathcal{A}(\sampleP)(1) = 0] \leq \eps / 4\,.\] 
    Similarly, with probability at least $1-\nfrac\delta4$, $ \cD_+(2) \mathds{1}[\mathcal{A}(\sampleP)(2) = 0] \leq \eps / 4$.

    Combining these two bounds with a union bound and \eqref{eq:thm:proper-intersect-closed-learnable-fn}, we get that, with probability at least $1-\nfrac\delta2$, the false-negative rate is bounded by
    \begin{equation} \label{eq:thm:proper-intersect-closed-learnable-bound-fn}
        \cD(\ell \setminus \mathcal{A}(\sampleP)) \leq \frac{3\eps}{4}\,.
    \end{equation}
    \noindent It remains to bound the false-positive rate. For every $k \in \{0, 1,  \ldots , n\}$, let $S_k$ denote the multi-set consisting of $k$ copies of $1$ and $n-k$ copies of $2$. By the definition of $\cA$, since $n\geq 9$,
    \begin{equation} \label{eq:thm:proper-intersect-closed-learnable-fp}
    \begin{aligned}
    \cD(\cA(\sampleP) \setminus \ell) = \sum_{k \in \Nat \setminus \ell} \cD(k) \mathds{1} \insquare{\cA(\sampleP)\inparen{k} = 1}  = \sum_{\sqrt{n} \leq k \leq n - \sqrt{n}, k \neq i^\star} \cD(k) \mathds{1} \insquare{\sampleP = S_k}\,.
    \end{aligned}
    \end{equation}
    
    \noindent Since $\cD_+(1),\cD_+(2)>\nfrac{1}{2\sqrt n}$, we have
    $\sup_{k\in\Nat}\cD_+(k)<1-\nfrac{1}{2\sqrt n}$. By \cref{lem:mode-n}, and since
    $n\geq \left(\nfrac{16C}{\eps\delta}\right)^4$, for every multi-set $S$ we have
    \[
        \Pr_{\sampleP\sim\cD_+^n}[\sampleP=S]
        \leq
        \sqrt{\frac{2C}{\sqrt n}}
        \leq
        \frac{\eps\delta}{8}\,.
    \]
    \noindent Combining this with \eqref{eq:thm:proper-intersect-closed-learnable-fp}, we get
    \[
    \begin{aligned}
        \E{\cD(\cA(\sampleP)\setminus\ell)}
        &\leq
        \sum_{\sqrt n \leq k \leq n-\sqrt n,\ k\neq i^\star}
        \cD(k)\Pr[\sampleP=S_k]
        \leq
        \frac{\eps\delta}{8}
        \sum_{\sqrt n \leq k \leq n-\sqrt n,\ k\neq i^\star}
        \cD(k)
        \leq
        \frac{\eps\delta}{8}\,.
    \end{aligned}
    \]
    \noindent Therefore, by Markov's inequality,
    $\Pr[\cD(\cA(\sampleP)\setminus\ell)>\eps/4]\leq \delta/2.$
    Combining this with \cref{eq:thm:proper-intersect-closed-learnable-bound-fn} shows that
    $\Pr_{\sampleP\sim\cD_+^n}[\err_{\cD}(\cA(\sampleP),\ell)>\eps]\leq \delta$, completing the proof of learnability by deterministic learners.

    \paragraph{Part 2 (Non-learnability by deterministic ERM learners).} We then prove that $\cH_{\mathrm{spr}}$ is not properly positive-only learnable by any deterministic ERM proper learner $\cA$. Fix any $n \geq 2$. For any $a \in \Nat \setminus \{1,2\}$, define the pair $(\cD_a,\ell_a)$ as follows:
    \[
        \ell_a=\inbrace{1,2,a+1}, \quad \cD_a(1)=\frac{1}{2n}, \quad \cD_a(2)=\frac{n-1}{2n}, \quad \cD_a(a)=\frac{1}{2}\,,
    \]
    where $\cD_a$ is zero elsewhere. Let $S$ be a multi-set with exactly one copy of $1$ and $n-1$ copies of $2$. Since $\cA$ is an ERM learner, we have $\Domain(S)\subseteq \cA(S)$. Thus, $\cA(S)=\{1,2,a^\star\}$ for some $a^\star\in\Nat\setminus\{1,2\}$. Furthermore, under the target $\ell_{a^\star}=\inbrace{1,2,a^\star+1}$, the point $a^\star$ is negative and has $\cD_{a^\star}$-mass $1/2$. Hence,
    \[
        \err_{\cD_{a^\star}}\inparen{\cA(S),\ell_{a^\star}}=\frac{1}{2}\,.
    \]
    \noindent In addition, under $\cD_{+,a^\star}$, the mass of $1$ is $1/n$ and the mass of $2$ is $(n-1)/n$. Therefore,
    \[
        \Pr_{\sampleP\sim \cD_{+,a^\star}^n}[\sampleP=S] \geq n\cdot \frac{1}{n}\inparen{\frac{n-1}{n}}^{n-1} > \inparen{\frac{n-1}{n}}^n \geq \frac{1}{4}\,.
    \]
    \noindent Consequently,
    \[
        \Pr_{\sampleP \sim \cD^{n}_{+,a^\star}}\insquare{\err_{\cD_{a^\star}}\inparen{\cA(\sampleP),\ell_{a^\star}} \geq \frac{1}{2}} \geq \frac{1}{4}\,. 
    \]
    \noindent Thus, $\cH_{\mathrm{spr}}$ is not properly positive-only learnable by $\cA$. \qedhere{}
    \end{proof}

\subsection{Proof of \cref{thm:nonuniform-separation}} \label{sec:nonuniform-separation}
    In this section, we prove \cref{thm:nonuniform-separation}, which we restate below.

\SepNonUnif*
\begin{proof}
We divide the proof into two parts corresponding to the two claims.

\paragraph{Part 1 (Separation of consistency and non-uniform learnability).} Consider the concept class
\[
    \cH_{\mathrm{bin}} \coloneqq \{ h_a \mid a \geq 3 \}\,,
    \qquad
    h_a \coloneqq \{1,2,a,a+2,a+4,\ldots\}\,.
\]
\noindent Observe that $\vc(\cH_{\mathrm{bin}})=1$. We first show that
$\cH_{\mathrm{bin}}$ satisfies the finite exterior separability property.
Consider any finite-sized $S \subseteq \Nat$ that is realized by
$\cH_{\mathrm{bin}}$, and any
$F \subseteq \Nat \setminus \CLOS_{\cH_{\mathrm{bin}}}(S)$.
We consider two cases for $S$, and in each case we show that there exists
$h \in \cH_{\mathrm{bin}}$ such that
$\Domain(S) \subseteq h$, but $F \cap h = \varnothing$.

\emph{Case 1:} There exists some $a \geq 3$ such that
$a \in \Domain(S)$. Denote
\[
    t = \min(\Domain(S)\setminus \{1,2\})\,.
\]
\noindent It is easy to see that
\(
    h_t = \CLOS_{\cH_{\mathrm{bin}}}(S).
\)
Thus, since $F \subseteq \Nat \setminus \CLOS_{\cH_{\mathrm{bin}}}(S)$, we have
$h_t \cap F = \varnothing$. Moreover, by the definition of the closure,
$\Domain(S) \subseteq h_t$.

\emph{Case 2:} $\Domain(S) \subseteq \{1,2\}$. If $F=\varnothing$, then any
$h \in \cH_{\mathrm{bin}}$ satisfies the desired condition. Otherwise, choose
$r \geq 3$ such that
$r > \max(F).$
\noindent Then $\Domain(S) \subseteq \{1,2\} \subseteq h_r$. Moreover, since
$r > \max(F)$, no element of $F$ belongs to the tail
$\{r,r+2,r+4,\ldots\}$ of $h_r$. Also, since
$\{1,2\} \subseteq \CLOS_{\cH_{\mathrm{bin}}}(S)$, we have
$\{1,2\} \cap F = \varnothing$. Therefore, $h_r \cap F = \varnothing$.

Combining these two cases, we conclude that $\cH_{\mathrm{bin}}$ satisfies
finite exterior separability. By \cref{thm:consistency-characterization},
$\cH_{\mathrm{bin}}$ is proper positive-only consistent. It remains to prove
that $\cH_{\mathrm{bin}}$ is not non-uniformly proper positive-only learnable.

For any $a \geq 3$, define the pair $(\cD_a,\ell_a)$ as follows:
\[
    \ell_a = h_{3+(a \bmod 2)}, \qquad
    \cD_a(1)=\cD_a(2)=\frac14, \qquad
    \cD_a(a)=\frac12\,,
\]
where $\cD_a$ is zero elsewhere. Notice that
\(
    \cD_{a,+} = U_{\{1,2\}}.
\)
Indeed, by the definition of $\ell_a$, the point $a$ is not labeled positive
by $\ell_a$, while both $1$ and $2$ are labeled positive.

For the sake of contradiction, suppose that $\cH_{\mathrm{bin}}$ is
non-uniformly proper positive-only learnable by a proper learner $\cA$.
Consider any
\[
    n \geq
    \max \left\{
        \mpos{}_{\cH_{\mathrm{bin}}}\left(\nfrac14,\nfrac14,h_3\right)\,,
        \mpos{}_{\cH_{\mathrm{bin}}}\left(\nfrac14,\nfrac14,h_4\right)
    \right\}\,.
\]
\noindent We now show that for any fixed $\sampleP \in \{1,2\}^n$,
\[
    \lim_{a \to \infty}
    \Pr_{\cA}
    \left[
        a \in \cA(\sampleP)
        \text{ or }
        a+1 \in \cA(\sampleP)
    \right]
    =1\,.
\]
\noindent Since $\cA$ is proper, for every fixed $\sampleP \in \{1,2\}^n$, there exists
an $\Nat_{\geq 3}$-valued random variable $B_{\sampleP}$ such that
$\cA(\sampleP)=h_{B_{\sampleP}}.$
For every realization $B_{\sampleP}=b$, if $a \geq b$, then exactly one of
$a$ and $a+1$ belongs to $h_b$, because $a$ and $a+1$ have opposite parity and
$h_b$ contains all sufficiently large numbers with the same parity as $b$.
Therefore,
\[
    \{B_{\sampleP} \leq a\}
    \subseteq
    \left\{
        a \in \cA(\sampleP)
        \text{ or }
        a+1 \in \cA(\sampleP)
    \right\}\,.
\]
\noindent Since $B_{\sampleP}$ is natural-number-valued, we have
$\lim_{a \to \infty} \Pr[B_{\sampleP} \leq a] = 1.$
Hence,
\[
    \lim_{a \to \infty}
    \Pr_{\cA}
    \left[
        a \in \cA(\sampleP)
        \text{ or }
        a+1 \in \cA(\sampleP)
    \right]
    =1\,.
\]
\noindent Since $\{1,2\}^n$ is finite, averaging over
$\sampleP \sim U_{\{1,2\}}^n$ gives
\[
    \lim_{a \to \infty}
    \Pr_{\sampleP \sim U_{\{1,2\}}^n,\,\cA}
    \left[
        a \in \cA(\sampleP)
        \text{ or }
        a+1 \in \cA(\sampleP)
    \right]
    =1\,.
\]

\noindent Now observe that the event $a \in \cA(\sampleP)$ implies
\[
    \err_{\cD_a}\left(\cA(\sampleP),\ell_a\right) \geq \frac12\,,
\]
because $\cD_a$ assigns mass $\frac12$ to $a$, and $a \notin \ell_a$.
Similarly, the event $a+1 \in \cA(\sampleP)$ implies
\[
    \err_{\cD_{a+1}}\left(\cA(\sampleP),\ell_{a+1}\right) \geq \frac12\,.
\]
\noindent Therefore,
\[
    \lim_{a \to \infty}
    \Pr_{\sampleP \sim U_{\{1,2\}}^n,\,\cA}
    \left[
        \err_{\cD_a}\left(\cA(\sampleP),\ell_a\right)
        +
        \err_{\cD_{a+1}}\left(\cA(\sampleP),\ell_{a+1}\right)
        \geq \frac12
    \right]
    =1\,.
\]
\noindent Consequently, there exists $a^\star \in \Nat$ such that
\[
    \Pr_{\sampleP \sim U_{\{1,2\}}^n,\,\cA}
    \left[
        \err_{\cD_{a^\star}}\left(\cA(\sampleP),\ell_{a^\star}\right)
        +
        \err_{\cD_{a^\star+1}}\left(\cA(\sampleP),\ell_{a^\star+1}\right)
        \geq \frac12
    \right]
    > \frac12\,.
\]
\noindent Using a union bound we get
\[
    \Pr_{\sampleP \sim U_{\{1,2\}}^n,\,\cA}
    \left[
        \err_{\cD_{a^\star}}\left(\cA(\sampleP),\ell_{a^\star}\right)
        \geq \frac14
    \right]
    +
    \Pr_{\sampleP \sim U_{\{1,2\}}^n,\,\cA}
    \left[
        \err_{\cD_{a^\star+1}}\left(\cA(\sampleP),\ell_{a^\star+1}\right)
        \geq \frac14
    \right]
    > \frac12\,.
\]
\noindent Thus, for some $b \in \{a^\star,a^\star+1\}$, we have
\[
    \Pr_{\sampleP \sim U_{\{1,2\}}^n,\,\cA}
    \left[
        \err_{\cD_b}\left(\cA(\sampleP),\ell_b\right)
        \geq \frac14
    \right]
    > \frac14\,.
\]
\noindent By definition, $\ell_b \in \{h_3,h_4\}$, and as observed above,
$\cD_{b,+}=U_{\{1,2\}}$. This contradicts 
\[
    n \geq
    \max \left\{
        \mpos{}_{\cH_{\mathrm{bin}}}\left(\nfrac14,\nfrac14,h_3\right)\,,
        \mpos{}_{\cH_{\mathrm{bin}}}\left(\nfrac14,\nfrac14,h_4\right)
    \right\}\,.
\]
\noindent Therefore, $\cH_{\mathrm{bin}}$ is not non-uniformly properly positive-only
learnable.

   \noindent \textbf{Part 2 (Separation of non-uniform learnability and learnability).} Recall the concept class considered in \cref{prop:des-strict}:
   \[
     \cH_{\mathrm{tail}} \coloneqq  \inbrace{h_a : a \geq 3}, \qquadwhere h_a \coloneqq  \inbrace{1,2} \cup \inbrace{m \in \N : m \geq a}\,.
   \]
   \noindent Observe that $\vc(\cH_{\mathrm{tail}}) = 1$. In \cref{prop:des-strict}, we
    already showed that $\cH_{\mathrm{tail}}$ does not satisfy the distributional exterior separability property, and thus it also does not satisfy the uniform exterior separability property. Therefore, by \cref{thm:random-characterization}, $\cH_{\mathrm{tail}}$ is not properly positive-only learnable. It remains to
    prove that $\cH_{\mathrm{tail}}$ is non-uniformly properly positive-only
    learnable.

    Define the learner $\cA$ as follows. For any $n \in \Nat$ and $\sampleP \in \cX^n$, let
    \[
        i(\sampleP)
        \coloneqq
        \min\bigl(\Domain(\sampleP)\setminus \inbrace{1,2}\bigr)\,,
    \]
    whenever $\Domain(\sampleP)\setminus \inbrace{1,2}$ is non-empty, and set
    \[
        \cA(\sampleP)
        \coloneqq
        \begin{cases}
            h_{i(\sampleP)}
            & \text{if } \Domain(\sampleP)\setminus \inbrace{1,2} \neq \varnothing,\\
            h_{\max\inbrace{n,3}}
            & \text{otherwise\,.}
        \end{cases}
    \]
    Fix a target concept $\ell=h_{a^\star}$ for some $a^\star \geq 3$. We show
    that this learner succeeds with a sample size that may depend on
    $a^\star$. Consider any $n \geq a^\star$. For every $\sampleP \in \ell^n$, we claim that
    \(
        \cA(\sampleP) \subseteq \ell.
    \)
    Indeed, if $\Domain(\sampleP)\setminus \inbrace{1,2}$ is non-empty, then
    $i(\sampleP) \in \ell\setminus \inbrace{1,2}$, and therefore
    $i(\sampleP)\geq a^\star$. Hence
    \(
        h_{i(\sampleP)} \subseteq h_{a^\star}.
    \)
    On the other hand, if
    $\Domain(\sampleP)\setminus \inbrace{1,2}=\varnothing$, then
    $\cA(\sampleP)=h_{n} \subseteq h_{a^\star}$. 
    Thus, for every positive-only sample $\sampleP \in \ell^n$, the learner has
    zero false positives.
    
    Next, observe that $\cA$ always satisfies $\Domain(\sampleP) \subseteq \cA(\sampleP)$. Therefore, by applying standard realizable PAC bounds to $\cH_{\mathrm{tail}}$ under $\D_+$, there exists a universal constant $C>0$ such that for
    \( n \geq C \frac{\log(1/\eps)+\log(1/\delta)}{\eps},\)
    with probability at least $1-\delta$, the false negative is bounded by
    \[
        \cD(\ell\setminus \cA(\sampleP)) \leq \cD_+(\ell\setminus \cA(\sampleP)) \leq \eps\,.
    \]
    \noindent Now using $\cA(\sampleP) \subseteq \ell$ yields that $\cH_{\mathrm{tail}}$ \mbox{is non-uniformly proper positive-only learnable by $\cA$.} \qedhere{}

\end{proof}

\subsection{Proof of \cref{thm:stable-characterization}} \label{sec:stable-characterization}

In this section, we prove \cref{thm:stable-characterization}.
The proof has two parts.
First, we show that exact exterior separability together with finite VC dimension is sufficient for
proper positive-only learning by a stable learner.
This direction is immediate from the closure-based learner.
Second, we prove a stronger necessity statement: exact exterior separability and finite VC dimension
are necessary even for a weaker notion of stability, which we call size-dependent stability.
We begin with sufficiency. %

\begin{proposition} \label{prop:ees-implies-stable}
    If $\cH$ satisfies exact exterior separation and $\vc{}\!\inparen{\cH}<\infty$, then $\cH$ is properly learnable from positive-only examples by a stable learner.
\end{proposition}
\begin{proof}
    Consider the learner $\mathcal{A}(\sampleP) = \CLOS_\cH(\sampleP)$. Clearly $\mathcal{A}$ satisfies $\mathcal{A}(S') = \mathcal{A}(S)$ for all positive samples $S$ and $S'$ such that $S \subseteq S'$ and $\Domain(S') \subseteq \mathcal{A}(S)$. Thus, $\mathcal{A}$ is stable. Combining this with \cref{lem:int-closed-pos-learnable} completes the proof.
\end{proof}

\noindent For the necessity direction, we prove a slightly stronger statement.
We show that exact exterior separability and bounded VC dimension are necessary even if the learner satisfies only the following
weaker stability condition.

\begin{definition} [Size-dependent Stability] \label{def:size-dependent-stable}
    We say that a proper learner $\mathcal A$ is \emph{size-dependent stable} if there exists a function $f:\cX^{[*]} \times \Nat \rightarrow \cH$ such that for all $S, S' \in \cX^{[*]}$ with probability 1 over any randomness in $\cA$, the following holds
        \[
            \text{if}\quad
            S\subseteq S' \text{ and } \Domain(S') \subseteq \cA(S) ,\quadtext{then}
            \cA(S')=  f(S, |S'|)\,.
        \]
        
\end{definition}

\begin{proposition} \label{prop:size-dependent-stable-implies-ees}
    If $\cH$ is properly learnable from positive-only examples by a size-dependent stable learner, then $\cH$ satisfies exact exterior separation and $\vc{}\!\inparen{\cH}<\infty$.
\end{proposition}
\begin{proof}
    Assume that $\cH$ does not satisfy the stated property. Then either $\VCD(\cH) = \infty$, which implies that $\VCD(\cH_{\cap}) = \infty$, and by the results of \citet{lee2026smoothed}, this in turn implies that $\cH$ is not positive-only learnable; or there exists a finite set $B = \{x_1, x_2, \ldots, x_s\} \subseteq \cX$ such that $\CLOS_\cH(B) \notin \cH$.
    
    Denote $F \coloneqq  \cX \setminus \CLOS_\cH(B)$. For any $x \in F$, define the distribution $\cD_x$ as
    \[
    \cD_x(x) = \frac{1}{2}\,, \quad \cD_{x} (x') = 
    \frac{1}{2s} \text{ for all } x' \in B\,, 
    \]
    and zero elsewhere. Also, let $\ell_x$ be any concept in $\cH$ such that $B \subseteq \ell_x$, but $\ell_x(x) = 0$.
    Note that since $x \notin \CLOS_\cH(B)$ such a concept $\ell_x$ always exists. Notice that for all $x \in F$, we have
    $\cD_{+, x} = U_B$. We prove that for every proper size-dependent stable positive-only learner $\mathcal{A}$,
    \[
    \lim_{n \rightarrow \infty} \sup_{x \in F} \Pr_{\sampleP \sim U_B^{n}, \cA} \left[  \err_{\cD_x} \left(\mathcal{A}(\sampleP), \ell_x \right) \geq \frac{1}{2s} \right] = 1\,.
    \]
    \noindent Therefore, $\mathcal A$ cannot learn $\cH$ with positive-only examples.

     Let $\mathcal{A}$ be any proper size-dependent stable positive-only learner for $\cH$. We prove the claim for two distinct cases, depending on the output of $\mathcal A$ over input samples that contain all $B$. 
     
     \paragraph{Case 1:} There exists a finite-sized multi-set $E \in \cX^{[*]}$ such that $\Domain(E) = B$, and 
     \[\Pr_{\cA}\insquare{B \subseteq \mathcal{A}(E)} > 0.\]
     Note that due to the stability of $\mathcal A$, every sample $\sampleP \in B^{n}$ that contains $E$ satisfies $\mathcal{A}(\sampleP) = f(E, n)$ with probability $1$ over the learner's randomness. Moreover, as $n$ goes to infinity the probability of a $\sampleP$ sampled by $U_B^{n}$ containing $E$ goes to 1. Thus, 
     \begin{equation} \label{eq:stable-proper-positive-char-1}
         \lim_{n \to \infty}\Pr_{\sampleP \sim U_B^{n}, \cA}[\mathcal{A}(\sampleP) = f(E, n)] = 1\,.
     \end{equation}
     
     \noindent We consider two distinct cases based on $f(E, n)$. \textbf{Case 1.1:} If $B \nsubseteq f(E, n)$, then for all $x \in F$ we have $\err_{\cD_x}(f(E, n), \ell_x) \geq \frac{1}{2s}$. By combining this with \eqref{eq:stable-proper-positive-char-1}, we attain our objective. \textbf{Case 1.2:}  If $B \subseteq f(E, n)$, then since $f(E, n) \in \cH$, there exists an $x_E \in f(E, n) \cap F$. Thus, $\err_{\cD_{x_E}}(f(E, n), \ell_{x_E}) = \nfrac{1}{2}$. By combining this with \eqref{eq:stable-proper-positive-char-1}, we attain our objective.

     \paragraph{Case 2:} No such $E$ exists. Then for every sample $\sampleP$ containing $B$ and every $x \in F$, we have 
     \[
     \Pr_{\sampleP \sim U_B^{n}, \cA}[\err_{\cD_x}(\mathcal{A}(\sampleP), \ell_x) \geq \nfrac{1}{2s}] = 1\,.
     \] 
     \noindent Note that, as $n$ goes to infinity, the probability of a $\sampleP$ sampled by $U_B^{n}$ containing $B$ goes to 1. \qedhere{}

    \end{proof}

    \noindent Combining \cref{prop:ees-implies-stable} and \cref{prop:size-dependent-stable-implies-ees} proves
\cref{thm:stable-characterization}.

\end{document}

%% file: macros.tex
\usepackage[utf8]{inputenc} %
\usepackage[T1]{fontenc}    %
\usepackage{hyperref}       %
\usepackage{url}            %
\usepackage{booktabs}       %
\usepackage{amsfonts}       %
\usepackage{nicefrac}       %
\usepackage{microtype}      %
\usepackage{xcolor}         %
\usepackage{mathrsfs}  
\usepackage{comment}

\usepackage{graphicx} %
\usepackage{amsthm}
\usepackage{amsmath,amsfonts,amssymb,mathtools, bbm}
\usepackage[linesnumbered,ruled,vlined]{algorithm2e}
\usepackage{xcolor}
\usepackage{multirow}

\usepackage{dsfont}
\usepackage{pifont} %
\usepackage{thm-restate} %
\usepackage[capitalise,noabbrev,nameinlink,sort]{cleveref} %

\newtheorem{theorem}{Theorem}%
\newtheorem{lemma}[theorem]{Lemma}
\newtheorem{proposition}[theorem]{Proposition}

\newtheorem{corollary}[theorem]{Corollary}

\newtheorem{defn}{Theorem}
\newtheorem{definition}[defn]{Definition}

\DeclareMathOperator{\Domain}{\mathrm{Domain}}

\newcommand{\sampleP}{S_+}

\newcommand{\D}{\mathcal{D}}
\newcommand{\Nat}{\mathbb{N}}

\newcommand{\E}[1]{\mathbb{E}\left[#1\right]}
\newcommand{\cE}[2]{\mathbb{E}_{#1}\left[#2\right]}

\newcommand{\VCD}{\operatorname{VCdim}}
\newcommand{\CLOS}{\operatorname{CLOS}}

\newcommand{\err}{\operatorname{err}}

\AfterEndEnvironment{definition}{\noindent\ignorespaces}
\AfterEndEnvironment{infdefinition}{\noindent\ignorespaces}
\AfterEndEnvironment{assumption}{\noindent\ignorespaces}
\AfterEndEnvironment{problem}{\noindent\ignorespaces}
\AfterEndEnvironment{openproblem}{\noindent\ignorespaces}
\AfterEndEnvironment{lemma}{\noindent\ignorespaces}
\AfterEndEnvironment{theorem}{\noindent\ignorespaces}
\AfterEndEnvironment{proposition}{\noindent\ignorespaces}
\AfterEndEnvironment{fact}{\noindent\ignorespaces}
\AfterEndEnvironment{question}{\noindent\ignorespaces}
\AfterEndEnvironment{corollary}{\noindent\ignorespaces}
\AfterEndEnvironment{model}{\noindent\ignorespaces}
\AfterEndEnvironment{remark}{\noindent\ignorespaces}
\AfterEndEnvironment{proof}{\noindent\ignorespaces}
\AfterEndEnvironment{fact}{\noindent\ignorespaces}
\AfterEndEnvironment{minftheorem}{\noindent\ignorespaces}
\AfterEndEnvironment{inftheorem}{\noindent\ignorespaces}
\AfterEndEnvironment{maintheorem}{\noindent\ignorespaces}
\AfterEndEnvironment{restatable}{\noindent\ignorespaces}
\AfterEndEnvironment{infassumption}
{\noindent\ignorespaces}
\AfterEndEnvironment{infcorollary}{\noindent\ignorespaces}

\crefname{section}{Section}{Sections}
\crefname{theorem}{Theorem}{Theorems}
\crefname{lemma}{Lemma}{Lemmas}
\crefname{problem}{Problem}{Problems}
\crefname{program}{Program}{Progams}
\crefname{definition}{Definition}{Definitions}
\crefname{conjecture}{Conjecture}{Conjectures}
\crefname{corollary}{Corollary}{Corollaries}
\crefname{construction}{Construction}{Constructions}
\crefname{conjecture}{Conjecture}{Conjectures}
\crefname{claim}{Claim}{Claims}
\crefname{observation}{Observation}{Observations}
\crefname{proposition}{Proposition}{Propositions}
\crefname{fact}{Fact}{Facts}
\crefname{question}{Question}{Questions}
\crefname{problem}{Problem}{Problems}
\crefname{remark}{Remark}{Remarks}
\crefname{example}{Example}{Examples}
\crefname{equation}{Equation}{Equations}
\crefname{appendix}{Section}{Sections}
\crefname{algorithm}{Algorithm}{Algorithms}
\crefname{model}{Model}{Models}
\crefname{figure}{Figure}{Figures}
\crefname{infassumption}{Informal Assumption}{Informal Assumptions}
\crefname{inftheorem}{Informal Theorem}{Informal Theorems}
\crefname{infdefinition}{Informal Definition}{Informal Definitions}
\crefname{minftheorem}{Main Informal Theorem}{Main Informal Theorems}
\crefname{maintheorem}{Main Theorem}{Main Theorems}
\crefname{assumption}{Assumption}{Assumptions}
\crefname{step}{Step}{Steps}
\crefname{result}{Result}{Results}
\crefname{event}{Event}{Events}
\crefname{none}{}{}

\usepackage{enumitem}

\newlist{asmpenum}{enumerate}{1} %
\setlist[asmpenum]{label={\arabic*.},ref=\theassumption.{\arabic*}}
\crefname{asmpenumi}{Assumption}{Assumptions}

\newcommand{\ie}{\textit{i.e.}}
\newcommand{\eg}{\textit{e.g.}}

\newcommand{\wh}[1]{\widehat{#1}}

\newcommand{\wt}[1]{\widetilde{#1}}

\newcommand{\customcal}[1]{\mathcal{#1}}
\newcommand{\cA}{\customcal{A}}

\newcommand{\cD}{\customcal{D}}

\newcommand{\cH}{\customcal{H}}

\newcommand{\cR}{\customcal{R}} 
\newcommand{\cS}{\customcal{S}}

\newcommand{\cX}{\customcal{X}}

\renewcommand{\d}{{\rm d}}

\makeatletter
\newcommand{\tagnum}[2]{%
    \refstepcounter{equation}%
    \tag{#1) \ (\theequation}%
    \protected@write \@auxout {}{%
        \string \newlabel {#2}{{\theequation}{\thepage}{}{equation.\theequation}{}}%
    }%
}
\makeatother

\makeatletter
\renewcommand{\eqref}[1]{\textup{\eqrefform@{\ref{#1}}}}
\let\eqrefform@\tagform@

\newcommand{\changetag}[1]{%
  \renewcommand\tagform@[1]{\maketag@@@{(\ignorespaces#1\unskip\@@italiccorr)}}%
}
\makeatother

\newcommand{\quadtext}[1]{\quad\text{#1}\quad}
\newcommand{\qquadtext}[1]{\qquad\text{#1}\qquad}
 
\newcommand{\quadand}{\quadtext{and}}
\newcommand{\qquadand}{\qquadtext{and}}

\newcommand{\qquadwhere}{\qquadtext{where}}

\usepackage[dvipsnames,table]{xcolor}
\usepackage{color} 
\definecolor{niceRed}{RGB}{190,38,38}
\definecolor{niceYellow}{HTML}{f5b400}
\definecolor{blueGrotto}{HTML}{059DC0}
\definecolor{royalBlue}{HTML}{057DCD}
\definecolor{navyBlue}{HTML}{0B579C}
\definecolor{yaleBlue}{HTML}{00356b}
\definecolor{limeGreen}{HTML}{81B622}
\definecolor{nicePurple}{HTML}{9c27b0}
\definecolor{lightRoyalBlue}{HTML}{def2ff}  
\definecolor{gold}{HTML}{ffa300}

\usepackage{color-edits}
\addauthor[Anay]{am}{gold}
\addauthor[Farnam]{fm}{red}
\addauthor[Shai]{sb}{blue}
\addauthor[Manolis]{mz}{teal}

\newcommand{\N}{\mathbb{N}}
\newcommand{\R}{\mathbb{R}}

\newcommand{\zo}{\ensuremath{\inbrace{0, 1}}}

\newcommand{\sfrac}[2]{{#1/#2}} 
\newcommand{\nfrac}[2]{\nicefrac{#1}{#2}}

\newcommand{\iid}{i.i.d.}

\def\abs#1{\left| #1 \right|}

\newcommand{\sinparen}[1]{(#1)}

\newcommand{\inbrace}[1]{\left\{#1\right\}}

\newcommand{\inparen}[1]{\left(#1\right)}
\newcommand{\insquare}[1]{\left[#1\right]}

\newcommand{\Cl}{\CLOS}
\newcommand{\eps}{\varepsilon}
\newcommand{\VC}{\VCD}
\newcommand{\vc}{\VCD}

\newcommand{\hstar}{h^\star}

\newcolumntype{L}[1]{>{\raggedright\let\newline\\\arraybackslash\hspace{0pt}}m{#1}}
\newcolumntype{C}[1]{>{\centering\let\newline\\\arraybackslash\hspace{0pt}}m{#1}}
\newcolumntype{R}[1]{>{\raggedleft\let\newline\\\arraybackslash\hspace{0pt}}m{#1}}

\newcommand{\hypo}[1]{\mathcal{#1}}

\newcommand{\hyH}{\hypo{H}}